\documentclass{article}

\usepackage{microtype}
\usepackage{graphicx}
\usepackage{booktabs} 

\usepackage{hyperref}

\usepackage[accepted]{icml2024}

\usepackage{amsmath}
\usepackage{amssymb}
\usepackage{mathtools}
\usepackage{amsthm}

\usepackage[capitalize,noabbrev]{cleveref}

\theoremstyle{plain}

\usepackage[textsize=tiny]{todonotes}
\usepackage{mathtools}

 \newcommand\independent{\protect\mathpalette{\protect\independenT}{\perp}}
    \def\independenT#1#2{\mathrel{\rlap{$#1#2$}\mkern2mu{#1#2}}}
\newcommand\notindependent{\!\perp\!\!\!\!\not\perp\!}

\usepackage{amsmath}
\usepackage{amsthm}
\makeatletter
\newtheorem*{rep@theorem}{\rep@title}
\newcommand{\newreptheorem}[2]{%
\newenvironment{rep#1}[1]{%
 \def\rep@title{#2 \ref{##1}}%
 \begin{rep@theorem}}%
 {\end{rep@theorem}}}
\makeatother
\usepackage{thmtools}
\usepackage{thm-restate}

\newtheorem{proposition}{Proposition}

\newtheorem{assumption}{Assumption}

\newtheorem{lemma}{Lemma}
\newtheorem{example}{Example}
\newtheorem{remark}{Remark}
\newreptheorem{proposition}{Proposition}
\usepackage{multirow,tikz,amsmath,comment}
\usetikzlibrary{arrows,fit,positioning,shapes,bayesnet}
\usetikzlibrary{automata}

\usepackage{enumitem}

\newcommand{\rebuttal}[1]{}
\usepackage{subcaption}

\usepackage{thmtools,thm-restate}

\newdimen\arrowsize
\pgfarrowsdeclare{arcsq}{arcsq}
{
	\arrowsize=0.2pt
	\advance\arrowsize by .5\pgflinewidth
	\pgfarrowsleftextend{-4\arrowsize-.5\pgflinewidth}
	\pgfarrowsrightextend{.5\pgflinewidth}
}
{
	\arrowsize=1.5pt
	\advance\arrowsize by .5\pgflinewidth
	\pgfsetdash{}{0pt} 
	\pgfsetroundjoin   
	\pgfsetroundcap    
	\pgfpathmoveto{\pgfpoint{0\arrowsize}{0\arrowsize}}
	\pgfpatharc{-90}{-140}{4\arrowsize}
	\pgfusepathqstroke
	\pgfpathmoveto{\pgfpointorigin}
	\pgfpatharc{90}{140}{4\arrowsize}
	\pgfusepathqstroke
}

\tikzset{
    block/.style={
        circle, draw, fill=white
        },
    myarrow/.style={
        single arrow,  
        draw, 
        single arrow head extend=2mm, minimum width=30pt,
        },    
    myar/.style={
        rounded corners=2pt, fill=black!20, 
        },
    mytri/.style={
        isosceles triangle, anchor=apex,
        isosceles triangle apex angle=90,
        minimum width=50pt
        },
    }

\icmltitlerunning{Causal Representation Learning from Multiple Distributions: A General Setting}

\begin{document}

\twocolumn[
\icmltitle{Causal Representation Learning from Multiple Distributions: A General Setting}



\icmlsetsymbol{equal}{*}

\begin{icmlauthorlist}
\icmlauthor{Kun Zhang}{equal,1,2}
\icmlauthor{Shaoan Xie}{equal,1}
\icmlauthor{Ignavier Ng}{equal,1}
\icmlauthor{Yujia Zheng}{1}
\end{icmlauthorlist}

\icmlaffiliation{1}{Carnegie Mellon University}
\icmlaffiliation{2}{Mohamed bin Zayed University of Artificial Intelligence}

\icmlkeywords{Machine Learning, ICML}

\vskip 0.3in
]



\printAffiliationsAndNotice{\icmlEqualContribution} 

\begin{abstract}
In many problems, the measured variables (e.g., image pixels) are just mathematical functions of the latent causal variables (e.g., the underlying concepts or objects). For the purpose of making predictions in changing environments or making proper changes to the system, it is helpful to recover the latent causal variables $Z_i$ and their causal relations represented by graph $\mathcal{G}_Z$. This problem has recently been known as causal representation learning. This paper is concerned with a general, completely nonparametric setting of causal representation learning from multiple distributions (arising from heterogeneous data or nonstationary time series), without assuming hard interventions behind distribution changes. 
We aim to develop general solutions in this fundamental case; as a by product, this helps see the unique benefit offered by other assumptions such as parametric causal models or hard interventions.  
We show that under the sparsity constraint on the recovered graph over the latent variables and suitable sufficient change conditions on the causal influences, interestingly, one can recover the moralized graph of the underlying directed acyclic graph, and the recovered latent variables and their relations are related to the underlying causal model in a specific, nontrivial way. In some cases, most latent variables can even be recovered up to component-wise transformations. Experimental results verify our theoretical claims.  
\end{abstract}

\section{Introduction}
Causal representation learning holds paramount significance across numerous fields, offering insights into intricate relationships within datasets. Most traditional methodologies (e.g., causal discovery) assume the observation of causal variables. This assumption, however reasonable, falls short in complex scenarios involving indirect measurements, such as electronic signals, image pixels, and linguistic tokens. Moreover, there are usually changes on the causal mechanisms in real-world, such as the heterogeneous or nonstationary data. Identifying the latent causal variables and their structures together with the change of the causal mechanism is in pressing need to understand the complicated real-world causal process. This has been recently known as causal representation learning \cite{Scholkopfetal21}. 

\looseness=-1
It is worth noting that identifying only the latent causal variables but not the structure among them, is already a considerable challenge. In the i.i.d. case, different latent representations can explain the same observations equally well, while not all of them are consistent with the true causal process. For instance, nonlinear independent component analysis (ICA), where a set of observed variables $X$ is represented as a mixture of independent latent variables $Z$, i.e, $X = g(Z)$, is known to be unidentifiable without additional assumptions \citep{Comon94}. While being a strictly easier task since there are no relations among latent variables, the identifiability of nonlinear ICA often relies on conditions on distributional assumptions (non-i.i.d. data) \citep{hyvarinen2016unsupervised, hyvarinen2017nonlinear, hyvarinen2019nonlinear, khemakhem2020variational, sorrenson2020disentanglement, lachapelle2021disentanglement, halva2020hidden, halva2021disentangling, yao2022temporally} or specific functional constraints \citep{Comon94, hyvarinen1999nonlinear, taleb1999source, buchholz2022function, zhengidentifiability, zheng2023generalizing}.

\looseness=-1
To generalize beyond the independent latent variables and achieve causal representation learning (recovering the latent variables and their causal structure), recent advances either introduce additional experiments in the forms of interventional or counterfactual data, or place more restrictive parametric or graphical assumptions on the latent causal model. For observational data, various graphical conditions have been proposed together with parametric assumptions such as linearity \citep{Silva06, cai2019triad, xie2020generalized, xie2022identification, adams2021identification, huang2022latent} and discreteness \citep{kivva2021learning}. For interventional data, single-node interventions have been considered together with parametric assumptions (e.g., linearity) on the mixing function \citep{varici2023score, ahuja2023interventional, buchholz2022function} or also on the latent causal model \cite{squires2023linear}. The nonparametric settings for both the mixing function and causal model have been explored by \citep{brehmer2022weakly, von2023nonparametric, jiang2023learning} together with additional assumptions on counterfactual views \citep{brehmer2022weakly}, distinct paired interventions \citep{von2023nonparametric}, and graphical conditions \citep{jiang2023learning}.

Despite the exciting developments in the field, one fundamental question pertinent to causal representation learning from multiple distributions remains unanswered--in the most general situation, without assuming parametric models on the data-generating process or the existence of hard interventions in the data, what information of the latent variables and the latent structure can be recovered? This paper attempts to provide an answer to it, which, surprisingly, shows that each latent variable can be recovered up to clearly defined indeterminacies. It suggests what we can achieve in the general case and furthermore, what unique contribution the typical assumptions that are currently made in causal representation learning from multiple distributions make towards complete identifiability of the latent variables (up to component-wise transformations). This may make it possible to figure out what minimal assumptions are needed to achieve complete identifiability, given partial knowledge of the system.\looseness=-1

\textbf{Contributions.} \ \ 
Concretely, as our contributions, we show that under the sparsity constraint on the recovered graph over the latent variables and suitable sufficient change conditions on the causal influences, interestingly, one can recover the moralized graph of the underlying directed acyclic graph (\cref{theorem:identifiabiltiy_markov_network}), and the recovered latent variables and their relations are related to the underlying causal model in a specific, nontrivial way (\cref{theorem:identifiabiltiy_causal_variables})--each latent variables is recovered as a function of itself and its so-called {\it intimate neighbors} in the Markov network implied by the true causal structure over the latent variables. Depending on the properties of the true causal structure over latent variables, the set of intimate neighbors might even be empty, in which case the corresponding latent variables can be recovered up to component-wise transformations (Remark \ref{remark:component_wise_identifiability}). Lastly, we show how the recovered moralized graph relates to the underlying causal graph under new relaxations of faithfulness assumption (\cref{proposition:moral_graph}).  
Simulation studies verified our theoretical findings.\looseness=-1

\section{Problem Setting}

\begin{figure}[t]
\centering
	{\hspace{0cm}\begin{tikzpicture}[scale=.75, line width=0.5pt, inner sep=0.2mm, shorten >=.1pt, shorten <=.1pt]
		\draw (0, 0) node(2) [circle, draw]  {{\footnotesize\,${Z}_4$\,}};
		\draw (-2, 0) node(1)[circle, draw]  {{\footnotesize\,${Z}_2$\,}};
  \draw (-1.4, 0.8) node(8)[circle, draw]  {{\footnotesize\,${Z}_3$\,}};
		\draw (2, 0) node(3)[circle, draw]  {{\footnotesize\,${Z}_5$\,}};
		\draw (-4, 0) node(5)[circle, draw]  {{\footnotesize\,$Z_1$\,}};
		\draw (-4.5, 1.7) node(9)  {{\footnotesize\,$\theta_1$\,}};
		\draw (-1.8, 1.7) node(10)  {{\footnotesize\,$\theta_3$\,}};
		\draw (-2.5, 1.7) node(11)  {{\footnotesize\,$\theta_{2}$\,}};
		\draw (-0.5, 1.7) node(12)  {{\footnotesize\,$\theta_4$\,}};
		\draw (1.5, 1.7) node(13)  {{\footnotesize\,$\theta_5$\,}};
		\draw[-arcsq] (5) -- (1); 
		\draw[-arcsq] (1) -- (2); 
		\draw[-arcsq] (8) -- (2); 
		\draw[-arcsq] (2) -- (3);
		\draw[-arcsq, dashed] (9) -- (5);
		\draw[-arcsq, dashed] (11) -- (1);
		\draw[-arcsq, dashed] (10) -- (8);
		\draw[-arcsq, dashed] (12) -- (2);
  	\draw[-arcsq, dashed] (13) -- (3);
\node[rectangle, fill=gray!20, inner sep=2.3mm, draw=black!100, opacity=.4, fit = (9) (5) (3) (13), yshift = .5mm](aa) {};
\node[myarrow, anchor=tip, minimum height=32pt, rotate=-90] at ([yshift=-42pt]aa.south) (bb) {$g$};
\node at ([xshift=9pt,yshift=-34pt]bb.south) {$X$};
		\end{tikzpicture}}
	\caption{The generating process for each latent causal variable $Z_i$ changes, governed by a latent factor $\theta_i$. The observed variables $X$ are generated by $X = g(Z)$ with a nonlinear mixing function $g$.}

	\vspace{-0.3em}
	\label{fig:setting}
\end{figure}
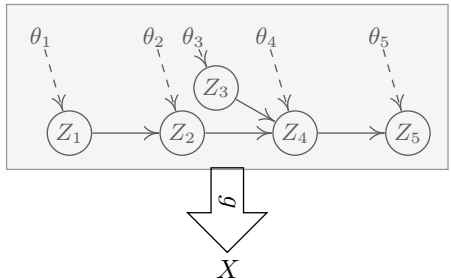

Let $X=(X_1,\dots,X_d)$ be a $d$-dimensional vector that represents the observed variables (e.g., image pixels). We assume that they are generated by $n$ latent causal variables $Z=(Z_1,\dots,Z_n)$ via a nonlinear mixing function $g:\mathbb{R}^n\rightarrow\mathbb{R}^d$ ($d \geq n$), which is a $\mathcal{C}^2$-diffeomorphism onto its image $\mathcal{X}\subseteq \mathbb{R}^d$. Furthermore, the variables $Z_i$'s are assumed to follow a structural equation model (SEM)~\citep{Pearl00}. Putting them together, the underlying data generating process can be written as
\begin{equation}\label{eq:data_generating_process}
\underbrace{X=g(Z)}_{\text{Nonlinear mixing}}, \quad \underbrace{Z_i = f_i(\textrm{PA}(Z_i), \epsilon_i; \theta_i),i=1,\dots,n}_{\text{Latent SEM}}.
\end{equation}
where $\textrm{PA}(Z_i)$ denotes the parents of variable $Z_i$, $\epsilon_i$'s are exogenous noise variables that are mutually independent, and $\theta_i$ denotes the latent (changing) factor (or effective parameters) associated with each model. Here, the data generating process of each latent variable $Z_i$ may change, e.g., across domains or over time, governed by the corresponding latent factor $\theta_i$; it is commonplace to encounter such changes in causal mechanisms in practice (arising from heterogeneous data or nonstationary time series). In addition, interventional data can be seen as a special type of change, which qualitatively restructure the causal relations. As their names suggest, we assume that the variables $X$ are observed, while the latent causal variables $Z$ and latent factors $\theta=(\theta_1,\dots,\theta_n)$ are unobserved. 

Let $P_{X;\theta}$ and $P_{Z;\theta}$ be the distributions of $X$ and $Z$, respectively, and their probability density functions be $p_X(X;\theta)$ and $p_Z(Z;\theta)$, respectively.\footnote{With a slight abuse of notation, we use the same capital letters $X$ and $Z$ to denote the variables and their values when the context is clear.} To lighten the notation, we drop the subscript in the density when the context is clear. The latent SEM in Eq. (\ref{eq:data_generating_process}) induces a causal graph $\mathcal{G}_Z$ with vertices $\{Z_i\}_{i=1}^n$ and edges $Z_j\rightarrow Z_i$ if and only if  {$Z_j\in\textrm{PA}(Z_i)$}. We assume that $\mathcal{G}_Z$ is acyclic, i.e., a directed acyclic graph (DAG). This implies that the distribution of variables $Z$ satisfy the Markov property w.r.t. DAG $\mathcal{G}_Z$ \citep{Pearl00}, i.e., $p(Z;\theta)=\prod_{i=1}^n p(Z_i\,|\,\textrm{PA}(Z_i);\theta_i)$. We provide an example of the data generating process in Eq.~(\ref{eq:data_generating_process}) and its corresponding latent DAG $\mathcal{G}_Z$ in Figure \ref{fig:setting}. Given samples of the observed variables $X$ arising from multiple distributions (or domains), say $\Theta=\{\theta^{(1)},\dots,\theta^{(m)}\}$, our goal is to recover the latent causal variables $Z$ and their causal relations up to minor indeterminacies. \looseness=-1

\section{Learning Causal Representations from Multiple Distributions}\label{sec:identifiability}
\vspace{-0.11em}

In this section, we provide theoretical results to show how one is able to recover the underlying latent causal variables and their causal relations up to certain indeterminacies from multiple distributions. Specifically, we show that under sparsity constraint on the recovered graph over the latent variables and suitable sufficient change conditions on the causal influences, the recovered latent variables are related to the true ones in a specific, nontrivial way. Such results serve as the foundation of our algorithm in Section \ref{sec:method}.

To start with, we estimate a model $(\hat{g}, \hat{f},p_{\hat{Z}},\hat{\Theta})$ that assumes the same data generating process as in Eq. (\ref{eq:data_generating_process}) and matches the true distribution of $X$ in different domains:
\begin{equation}\label{eq:matched_distribution}
\hspace{-0.9em}p_X(X';\theta^{(u)})=p_{\hat{X}}(X';\hat{\theta}^{(u)}), \:\,\forall \ \theta^{(u)}\in\Theta, \ X'\in \mathcal{X}^{(u)},
\end{equation}
where $\theta^{(u)}$ denotes the latent factor in the $u$-th domain, and $\mathcal{X}^{(u)}$ is the image of function $g$ in the $u$-th domain. Here, $X$ and $\hat{X}$ are generated by the true
model $(g, f,p_{Z},\Theta)$ and the estimated model $(\hat{g}, \hat{f},p_{\hat{Z}},\hat{\Theta})$, respectively.

A key ingredient of our results is the Markov network that represents conditional dependencies among random variables via an undirected graph. Let $\mathcal{M}_Z$ be the Markov network over variables $Z$, i.e., with vertices $\{Z_i\}_{i=1}^n$ and edges $\{Z_i,Z_j\}\in\mathcal{E}(\mathcal{M}_Z)$ if and only if $Z_i \notindependent Z_j \mid Z_{[n] \backslash \{i,j\}}$.\footnote{We use $[n]$ to denote $\{1,\dots,n\}$ and $Z_{[n] \backslash \{i,j\}}$ to denote $\{Z_i\}_{i=1}^{n}\setminus\{Z_i,Z_j\}$.} Also, we denote by $|\mathcal{M}_Z|$ the number of undirected edges in the Markov network. In Section \ref{sec:recover_markov_network}, apart from showing how to estimate the underlying latent causal variables up to certain indeterminacies, we also show that such latent Markov network $\mathcal{M}_Z$ can be recovered up to isomorphism. To achieve so, we make use of the following property (assuming that $p_Z$ is twice differentiable):
\vspace{-0.08em}
\begin{equation}
\label{eq:cross_de}
Z_i \independent Z_j \mid Z_{[n] \backslash \{i,j\}} \iff \frac{\partial^2 \log p(Z;\theta)}{\partial{Z_i} \partial{Z_j}} = 0.
\vspace{-0.08em}
\end{equation}
Such a connection between pairwise conditional independence and cross derivatives of the density function has been noted by \citet{lin1997factorizing} and utilized in Markov network learning for observed variables \citep{zheng2023generalized}. With the recovered latent Markov network structure, we provide results in Section \ref{sec:recover_dag} to show how it relates to the moralized graph of true latent causal DAG $\mathcal{G}_Z$, by exploiting a specific type of faithfulness assumption that is considerably weaker than the standard faithfulness assumption used in the literature of causal discovery \citep{Spirtes00}.

\vspace{-0.15em}
\subsection{Recovering Latent Causal Variables and Latent Markov Network}\label{sec:recover_markov_network}
\vspace{-0.05em}
We consider a general, completely nonparametric setting of causal representation learning from multiple distributions. Specifically, we show how one can recover the latent causal variables and the Markov network structure among them up to minor indeterminacies, by leveraging sparsity constraint and sufficient change conditions on the causal mechanisms. Notably, in some cases, most latent variables can even be recovered up to component-wise transformations.

We start with the following result that provides information about the derivative of true latent causal variables $Z$ with respect to the estimated ones $\hat{Z}$, according to their corresponding Markov networks $\mathcal{M}_{Z}$ and $\mathcal{M}_{\hat{Z}}$. Result of this form is often used in the proof of nonlinear ICA to obtain identifiability of component-wise nonlinear transformations~\citep{hyvarinen2016unsupervised,hyvarinen2019nonlinear}. At the same time, our result here is different from that of nonlinear ICA as it allows for causal relations among latent variables. This result serves as the backbone of our further identifiability results in this section.

\begin{restatable}{proposition}{ThmRelationMarkovNetwork}\label{proposition:relation_markov_network}
Let the observations be sampled from the data generating process in Eq. (\ref{eq:data_generating_process}), and $\mathcal{M}_Z$ be the Markov network over $Z$. Suppose  the following assumptions hold:
\begin{itemize}
\item A1 (Smooth and positive density): The probability density function of latent causal variables, i.e., $p_Z$, is twice continuously differentiable and positive in $\mathbb{R}^n$.
\item A2 (Sufficient changes): For each value of $Z$, there exist $2n+|\mathcal{M}_Z|+1$ values of $\theta$, i.e., $\theta^{(u)}$ with $u=0,\dots,2n+|\mathcal{M}_Z|$, such that the vectors $w(Z, {u})-w(z,0)$ with $u=1,\dots,2n+|\mathcal{M}_Z|$ are linearly independent, where vector $w(Z, {u})$ is defined as follows:\footnote{We denote by $\oplus$ the vector concatenation symbol. Also, the order in the mixed partial derivatives can be interchanged.}
\begin{flalign*}
&\hspace{-0.7em} w(Z, {u})=\left(\frac{\partial \log p(Z;\theta^{(u)})}{\partial  Z_i}\right)_{i\in[n]}\\
&\hspace{-0.7em} \quad\qquad\qquad \oplus\left(\frac{\partial^2 \log p(Z;\theta^{(u)})}{\partial  Z_i^2}\right)_{i\in[n]}\\
&\hspace{-0.7em} \quad\qquad\qquad \oplus \left(\frac{\partial^2 \log p(Z;\theta^{(u)})}{\partial  Z_i \partial  Z_j}\right)_{\{Z_i,Z_j\}\in\mathcal{E}(\mathcal{M}_Z),\,i<j}.
\end{flalign*}
\end{itemize}
\vspace{-0.05em}
Suppose that we learn $(\hat{g}, \hat{f},p_{\hat{Z}},\hat{\Theta})$ to achieve Eq. (\ref{eq:matched_distribution}). Then, for every pair of estimated latent variables $\hat{Z}_k$ and $\hat{Z}_l$ that are {\bf not adjacent in the Markov network} $\mathcal{M}_{\hat{Z}}$ over $\hat{Z}$, we have the following statements:
\vspace{-0.05em}
\begin{enumerate}[label=(\alph*)]
\item For each true latent causal variable $Z_i$, we have
\begin{equation}\label{eq:relation_1}
\frac{\partial  Z_i}{\partial \hat{Z}_k}\frac{\partial  Z_i}{\partial \hat{Z}_l}=0.
\end{equation}
\item For each pair of true latent causal variables $Z_i$ and $Z_j$ that are adjacent in the Markov network $\mathcal{M}_Z$, we have 
\begin{equation}\label{eq:relation_2}
\frac{\partial  Z_i}{\partial \hat{Z}_k}\frac{\partial  Z_j}{\partial \hat{Z}_l}=0.
\end{equation}
\end{enumerate}
\end{restatable}

The proof is provided in Appendix \ref{sec:proof:relation_markov_network}, which leverages the property of Markov network in Eq.~\eqref{eq:cross_de}. Assumption A2 can be viewed as suitable sufficient change conditions on the causal influences across different domains. It is worth noting that the requirement of a sufficient number of domains has been commonly adopted in the literature (e.g., see \citet{hyvarinen2023nonlinear} for a recent survey), such as visual disentanglement \citep{khemakhem2020ice}, domain adaptation \citep{kong2022partial}, video analysis \citep{yao2021learning}, and image-to-image translation \citep{xie2022identification}. Also, we do not specify exactly how to learn $(\hat{g}, \hat{f},p_{\hat{Z}},\hat{\Theta})$ to achieve Eq. (\ref{eq:matched_distribution}), and leave the door open for different approaches to be used, such as normalizing flow and variational approaches. For example, we adopt a variational approach in Section \ref{sec:method}.

In \cref{theorem:partial_disentanglement}, Eqs. \eqref{eq:relation_1} and \eqref{eq:relation_2} hold for every sample of $Z$. Intuitively, one may expect that Eq.~\eqref{eq:relation_1} implies either $\frac{\partial  Z_i}{\partial \hat{Z}_k}=0$ for all samples of $Z$, or $\frac{\partial  Z_i}{\partial \hat{Z}_l}=0$ for all samples, i.e., the zero entries in the Jacobian matrix  (of the function from $\hat{Z}$ to $Z$)  remain in the same positions across different samples. If this conclusion holds true, it indicates that the true latent variable $Z_i$ cannot be a function of both estimated latent variables $\hat{Z}_k$ and $\hat{Z}_l$, which is helpful for disentanglement. The same reasoning applies to Eq. \eqref{eq:relation_2}. In fact, similar conclusion can often be obtained in the proof of identifiability for nonlinear ICA~\citep{hyvarinen2019nonlinear}, by leveraging the continuity and invertibility of the Jacobian matrix.\looseness=-1

However, this conclusion in general does not hold in our setting (that allows for causal relations among latent variables $Z$) without any constraint on the sparsity of recovered Markov network, for which counterexamples exist. The reason is that each of Eqs.~\eqref{eq:relation_1} and~\eqref{eq:relation_2} correspond to a pair of recovered latent variables $\hat{Z}$ that are not adjacent in the Markov network $\mathcal{M}_{\hat{Z}}$, and can be viewed as a specific form of restriction on the Jacobian matrix (of the function from $\hat{Z}$ to $Z$). When the recovered Markov network is relatively dense, less restrictions are imposed on the Jacobian matrix, and thus there are possibilities for the aforementioned zero entries to switch positions across different samples. Interestingly, incorporating sparsity constraint on the recovered Markov network during estimation can help eliminate these possibilities, formally described below.

\begin{restatable}[Relations among true and recovered latent causal variables]{theorem}{ThmPartialDisentanglement}\label{theorem:partial_disentanglement}
Let the observations be sampled from the data generating process in Eq. (\ref{eq:data_generating_process}), and $\mathcal{M}_Z$ be the Markov network over $Z$. Suppose that Assumptions A1 and A2 from Theorem 1 hold. Suppose also that we learn $(\hat{g}, \hat{f},p_{\hat{Z}},\hat{\Theta})$ to achieve Eq. (\ref{eq:matched_distribution})  {with
the minimal number of edges of the Markov network $\mathcal{M}_{\hat{Z}}$ over $\hat{Z}$}. Then, for every pair of estimated latent variables $\hat{Z}_k$ and $\hat{Z}_l$ that are {\bf not adjacent in the Markov network} $\mathcal{M}_{\hat{Z}}$ over $\hat{Z}$, we have the following statements:
\begin{enumerate}[label=(\alph*)]
\item Each true latent causal variable $Z_i$ is a function of at most one of $\hat{Z}_k$ and $\hat{Z}_l$.
\item For each pair of true latent causal variables $Z_i$ and $Z_j$ that are adjacent in the Markov network $\mathcal{M}_Z$ over $Z$, at most one of them is a function of  $\hat{Z}_k$ or $\hat{Z}_l$.
\end{enumerate}
\end{restatable}

The proof can be found in Appendix \ref{sec:proof:partial_disentanglement}. The above result sheds light on how each pair of the estimated latent variables $\hat{Z}_k$ and $\hat{Z}_l$ that are not adjacent in Markov network $\mathcal{M}_{\hat{Z}}$ relate to the true latent causal variables $Z$, thus providing information for further disentanglement. Furthermore, note that a trivial solution would be a complete graph over $\hat{Z}$ without any constraint on the estimating process. Apart from providing information for disentanglement, we show below that incorporating sparsity constraint on the recovered Markov network also helps avoid this trivial solution and recover the underlying Markov network up to isomorphism. The proof is given in Appendix \ref{sec:proof:identifiabiltiy_markov_network}.
\begin{restatable}[Identifiability of latent Markov network]{theorem}{ThmIdentifiabilityMarkovNetwork}\label{theorem:identifiabiltiy_markov_network}
Let the observations be sampled from the data generating process in Eq. (\ref{eq:data_generating_process}), and $\mathcal{M}_Z$ be the Markov network over $Z$. Suppose that Assumptions A1 and A2 from Theorem 1 hold. Suppose also that we learn $(\hat{g}, \hat{f},p_{\hat{Z}},\hat{\Theta})$ to achieve Eq. (\ref{eq:matched_distribution})  {with
the minimal number of edges of the Markov network $\mathcal{M}_{\hat{Z}}$ over $\hat{Z}$}. Then, the recovered latent Markov network $\mathcal{M}_{\hat{Z}}$ is isomorphic to the true latent Markov network $\mathcal{M}_{Z}$.
\end{restatable}

In addition to recovering the underlying Markov network $\mathcal{M}_Z$, we show that the sparsity constraint on the recovered Markov network $\mathcal{M}_{\hat{Z}}$ also allows us to recover the underlying latent causal variables $Z$ up to specific, relatively minor indeterminacies. In the result, the following variable set, termed {\it intimate neighbor set}, plays an important role:
\begin{multline} \nonumber
{\Psi}_{Z_i}: = \{ Z_j \,|\,  Z_j,j\neq i, \textrm{ is adjacent to }Z_i\textrm{ and}\\ \textrm{all other neighbors of }Z_i\textrm{ in }\mathcal{M}_Z \}.
\end{multline}
For example, according to the Markov network implied by $\mathcal{G}_Z$ in Figure \ref{fig:setting}, ${\Psi}_{Z_1} = \{Z_2\}$, ${\Psi}_{Z_2} = \varnothing$, where $\varnothing$ denotes the empty set, ${\Psi}_{Z_3} = \{Z_2, Z_4\}$, ${\Psi}_{Z_4} = \varnothing$, and ${\Psi}_{Z_5} = \{Z_4\}$.  As another example, according to the Markov network in Figure \ref{fig:example}(b), which is implied by the DAG in Figure \ref{fig:example}(a), we have ${\Psi}_{Z_i} = \varnothing$ for $i=1,2, 3,5,6$ and ${\Psi}_{Z_4} = \{Z_3, Z_6\}$.

\begin{restatable}[Identifiability of latent causal variables]{theorem}{ThmIdentifiabilityCausalVariables}\label{theorem:identifiabiltiy_causal_variables}
Let the observations be sampled from the data generating process in Eq. (\ref{eq:data_generating_process}). Suppose that Assumptions A1 and A2 from Theorem 1 hold. Suppose also that we learn $(\hat{g}, \hat{f},p_{\hat{Z}},\hat{\Theta})$ to achieve Eq. (\ref{eq:matched_distribution})  {with
the minimal number of edges of the Markov network $\mathcal{M}_{\hat{Z}}$ over $\hat{Z}$}. Then, there exists a permutation $\pi$ of the estimated latent variables, denoted as $\hat{Z}_{\pi}$, such that each $\hat{Z}_{\pi(i)}$ is solely a function of $Z_i$ and a subset of $\Psi_{Z_i}$.
\end{restatable}
Note that the term `solely' indicates that $\hat{Z}_{\pi(i)}$ is not a function of other variables that are not specified in the theorem above. The proof is given in Appendix \ref{sec:proof:identifiabiltiy_causal_variables}. Roughly speaking, the proof leverages \cref{theorem:partial_disentanglement,theorem:identifiabiltiy_markov_network} to reason about the relationships among the true latent variables and the recovered ones, which imply that certain entries on the Jacobian matrix $\frac{\partial Z}{\partial \hat{Z}}$ must be zero. We show that these entries remain zero in the powers of $\frac{\partial Z}{\partial \hat{Z}}$, indicating that the same entries remain zero in $\big(\frac{\partial Z}{\partial \hat{Z}}\big)^{-1}$ (because the inverse $\big(\frac{\partial Z}{\partial \hat{Z}}\big)^{-1}$ can be written as a linear combination of the powers of $\frac{\partial Z}{\partial \hat{Z}}$) and thus $\frac{\partial \hat{Z}}{\partial Z}$, from which the identifiability result can be derived.

It is worth noting that in many cases, \cref{theorem:identifiabiltiy_causal_variables} already enables us to recover some of the latent variables up to a component-wise transformation. 
\begin{remark}\label{remark:component_wise_identifiability}
No matter how many neighbors each latent causal variable $Z_i$ has, as long as each of its neighbors is not adjacent to at least one other neighbor in the Markov network $\mathcal{M}_Z$, then $Z_i$ can be recovered up to a component-wise transformation. 
\end{remark}
Even if the above case does not hold, \cref{theorem:identifiabiltiy_causal_variables} still shows how the estimated latent variables relate to the underlying causal variables in a specific, nontrivial way. Two examples are provided below.

\begin{figure}[t]
\centering
	\begin{subfigure}{0.23\textwidth}
	{\hspace{0cm}\begin{tikzpicture}[scale=.75, line width=0.5pt, inner sep=0.2mm, shorten >=.1pt, shorten <=.1pt]
		\draw (0, 0) node(2) [circle, draw]  {{\footnotesize\,${Z}_3$\,}};
		\draw (-1.4, 0) node(1)[circle, draw]  {{\footnotesize\,${Z}_2$\,}};
  \draw (-1.4, 1.2) node(8)[circle, draw]  {{\footnotesize\,${Z}_5$\,}};
		\draw (1.4, 0) node(3)[circle, draw]  {{\footnotesize\,${Z}_4$\,}};
		\draw (-2.8, 0) node(5)[circle, draw]  {{\footnotesize\,$Z_1$\,}};
		\draw (0, 1.2) node(9)[circle, draw]  {{\footnotesize\,$Z_6$\,}};
		\draw[-arcsq] (5) -- (1); 
		\draw[-arcsq] (1) -- (2); 
		\draw[-arcsq] (2) -- (3);
  \draw[-arcsq] (5) -- (8);
  \draw[-arcsq] (8) -- (9);
  \draw[-arcsq] (9) -- (3);
		\end{tikzpicture}}
  \caption{$\mathcal{G}_Z$,  the DAG over true latent variables $Z_i$.}
  \end{subfigure}
  ~
  \begin{subfigure}{0.23\textwidth}
	{\hspace{0cm}\begin{tikzpicture}[scale=.75, line width=0.5pt, inner sep=0.2mm, shorten >=.1pt, shorten <=.1pt]
		\draw (0, 0) node(2) [circle, draw]  {{\footnotesize\,${Z}_3$\,}};
		\draw (-1.4, 0) node(1)[circle, draw]  {{\footnotesize\,${Z}_2$\,}};
  \draw (-1.4, 1.2) node(8)[circle, draw]  {{\footnotesize\,${Z}_5$\,}};
		\draw (1.4, 0) node(3)[circle, draw]  {{\footnotesize\,${Z}_4$\,}};
		\draw (-2.8, 0) node(5)[circle, draw]  {{\footnotesize\,$Z_1$\,}};
		\draw (0, 1.2) node(9)[circle, draw]  {{\footnotesize\,$Z_6$\,}};
		\draw[-] (5) -- (1); 
		\draw[-] (1) -- (2); 
		\draw[-] (2) -- (3);
  \draw[-] (5) -- (8);
  \draw[-] (8) -- (9);
  \draw[-] (9) -- (3);
  \draw[-] (2) -- (9);
		\end{tikzpicture}}
  \caption{The corresponding Markov network $\mathcal{M}_Z$.}
  \end{subfigure}
  \vspace{-0.1em}
	\caption{Illustrative example 2.}
	\vspace{-0.2em}
	\label{fig:example}
\end{figure}
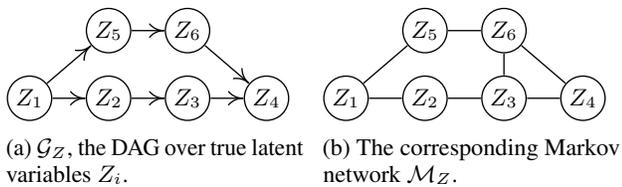

\begin{example}
Consider the Markov network $\mathcal{M}_Z$ corresponding to the DAG $\mathcal{G}_Z$ over $Z$ in Figure \ref{fig:setting}. By \cref{theorem:identifiabiltiy_causal_variables} and suitable permutation of estimated latent variables $\hat{Z}$, we have: (a) $\hat{Z}_{\pi(1)}$ is solely a function of a subset of $\{Z_1,Z_2\}$, (b) $\hat{Z}_{\pi(2)}$ is solely a function of $Z_2$, (c) $\hat{Z}_{\pi(3)}$ is solely a function of a subset of $\{Z_2,Z_3,Z_4\}$, (d) $\hat{Z}_{\pi(4)}$ is solely a function of $Z_4$, and (e) $\hat{Z}_{\pi(5)}$ is solely a function of a subset of $\{Z_4,Z_5\}$.
In this example, the latent causal variables $Z_2$ and $Z_4$ can be recovered up to component-wise transformation, while variables $Z_1$, $Z_3$, and $Z_5$ can be identified up to mixtures with certain neighbors in the Markov network.
\end{example}
\begin{example}
One may think that generally speaking, the more complex $\mathcal{G}_Z$, the more indeterminacies we have in the estimated latent variables (in the sense that each estimated latent variable receives contributions from more latent variables). In fact, this may not be the case.  For instance, consider the underlying latent causal graph $\mathcal{G}_Z$ in Figure \ref{fig:example}(a), which involves more variables and more edges and whose Markov network is shown in Figure \ref{fig:example}(b). For every variable $Z_i$ that is not the sink node, it has $\Psi_{Z_i} = \varnothing$ and thus can be recovered up to a component-wise transformation. \looseness=-1
\end{example}

\vspace{-0.25em}
\textbf{Permutation of recovered latent variables.} \ \
Theorems~\ref{theorem:identifiabiltiy_markov_network} and \ref{theorem:identifiabiltiy_causal_variables} involve certain permutation of the estimated latent variables $\hat{Z}$. Such an indeterminacy is common in the literature of causal discovery and representation learning tasks involving latent variables. In our case,  {since the function $v \coloneqq g^{-1}\circ \hat{g}$ where $Z = v(\hat{Z})$ is invertible, there exists a permutation of the latent variables such that the corresponding Jacobian matrix $J_v$ has nonzero diagonal entries (see \cref{lemma:nonzero_diagonal_entries} in Appendix \ref{app:proof:nonzero_diagonal_entries})}; such a permutation is what Theorems \ref{theorem:identifiabiltiy_markov_network} and \ref{theorem:identifiabiltiy_causal_variables} refer to. 

\vspace{-0.1em}
\textbf{Connection with nonlinear ICA.} \ \
It is worth noting that nonlinear ICA (with auxiliary variables) may be viewed as a special case of our result in this section. Specifically, if the true latent causal DAG $\mathcal{G}_Z$ is an empty graph, then the latent causal variables are independent, which reduce to the nonlinear ICA setting \citep{hyvarinen2019nonlinear}.

\vspace{-0.1em}
Furthermore, since traditional nonlinear ICA always has a valid solution (to produce nonlinear independent components) \cite{Hyv99}, one may wonder whether, in our setting, it is possible to always find nonlinear components as functions of $X$ that are independent in each domain, as produced by recent methods for nonlinear ICA with auxiliary variables \cite{hyvarinen2019nonlinear}. As a corollary of \cref{theorem:identifiabiltiy_markov_network}, we show that the answer is no--there do not exist nonlinear components that are independent across domains if the true latent causal DAG $\mathcal{G}_Z$ is not an empty graph. The proof is provided in Appendix \ref{sec:proof:identifiabiltiy_ICA}.
\begin{restatable}[Impossibility of finding independent components]{corollary}{ThmIdentifiabilityICA}\label{theorem:identifiabiltiy_ICA}
Let the observations be sampled from the data generating process in Eq. (\ref{eq:data_generating_process}). Suppose that Assumptions A1 and A2 from Theorem 1, as well as Assumptions \ref{assump:adjacency_faith} and \ref{assump:sucf}, hold, and that the true latent causal DAG $\mathcal{G}_Z$ is not an empty graph. Suppose also that we learn $(\hat{g}, \hat{f},p_{\hat{Z}},\hat{\Theta})$ with the components of $\hat{Z}$ being independent in each domain. Then, $(\hat{g}, \hat{f},p_{\hat{Z}},\hat{\Theta})$ cannot achieve Eq. (\ref{eq:matched_distribution}).
\end{restatable}

\vspace{-0.5em}
\subsection{From Latent Markov Network to Latent Causal DAG}\label{sec:recover_dag}
\vspace{-0.1em}

Now we have identified the Markov network up to an isomorphism, which characterizes conditional independence relations in the distribution. To build the connection between Markov network or conditional independence relations and causal structures, prior theory relies on the Markov and faithfulness assumptions. However, in real-world scenarios, the faithfulness assumption could be violated due to various reasons including path cancellations \citep{zhang2008detection,uhler2013geometry}.\looseness=-1

\vspace{-0.1em}
Since our goal is to generalize the identifiability theory as much as possible to fit practical applications, we introduce two relaxations of the faithfulness assumptions.

\begin{assumption}[Single adjacency-faithfulness (SAF)]
\label{assump:adjacency_faith}
Given a DAG $\mathcal{G}_Z$ and distribution $P_{Z;\theta}$ over the variable set $Z$, if two variables $Z_i$ and $Z_j$ are adjacent in $\mathcal{G}_Z$, then $Z_i \notindependent Z_j\mid Z_{[n] \setminus \{i,k\}}$.
\end{assumption}

\begin{assumption}[Single unshielded-collider-faithfulness (SUCF) \citep{Ng2021reliable}]
\label{assump:sucf}
Given a  latent causal graph $\mathcal{G}_Z$ and distribution $P_{Z;\theta}$ over the variable set $Z$, let $Z_i\rightarrow Z_j \leftarrow Z_k$ be any unshielded collider in $\mathcal{G}_Z$,
then $Z_i \notindependent Z_k \mid Z_{[n] \setminus \{i,k\}}$.
\end{assumption}

\looseness = -1
We propose SAF as a relaxtion of the Adjacency-faithfulness \citep{ramsey2012adjacency}. The SUCF assumption is first introduced by \citet{Ng2021reliable}, which is strictly weaker than Orientation-faithfulness \citep{ramsey2012adjacency}. Thus, both of them are strictly weaker than the faithfulness assumption, since the combination of Adjacency-faithfulness and Orientation-faithfulness is weaker than the faithfulness assumption \citep{zhang2008detection}.

Interestingly, not only they are weaker variants of faithfulness, but we also prove that they are actually \text{necessary and sufficient conditions}, thus the weakest possible ones, to bridge conditional independence relations and causal structures. Specifically, we show that the recovered Markov network (e.g., in \cref{theorem:identifiabiltiy_markov_network}) is exactly the moralized graph of the true causal DAG if and only if the proposed variants of faithfulness hold. The proofs of \cref{lem:moral_sub} and \cref{proposition:moral_graph} are shown in Appendix \ref{sec:proof_thm:moral_sub}.

\begin{restatable}{lemma}{LemMoralSub}\label{lem:moral_sub}
Given a latent causal graph $\mathcal{G}_Z$ and distribution $P_{Z;\theta}$ with its Markov Network $\mathcal{M}_Z$, under Markov assumption, the undirected graph defined by $\mathcal{M}_Z$ is a subgraph of the moralized graph of the true causal DAG $G$.
\end{restatable}
\begin{restatable}[Moralized graph and Markov network]{proposition}{ThmMoralGraph}\label{proposition:moral_graph}
Given a causal DAG $\mathcal{G}_Z$ and distribution $P_{Z;\theta}$ with its Markov Network $\mathcal{M}_Z$, under Markov assumption, the undirected graph defined by $\mathcal{M}_Z$ is the moralized graph of the true causal DAG $\mathcal{G}_Z$ if and only if the SAF and SUCF assumptions are satisfied.
\end{restatable}
It is worth noting that the connection between conditional independence relations and causal structures has been developed by \citep{loh2014high, Ng2021reliable} in the linear case by leveraging the properties of the inverse covariance matrix; our results here focus on the nonparametric case and thus being able to serve the considered general settings for identifiability. Also note that the necessary and sufficient conditions may also be of independent interest for other causal discovery tasks exploring conditional independence relations in the nonparametric case.

\textbf{Discussion on additional assumptions.} \ \ We investigated how the sparsity constraint on the recovered graph over latent variables and sufficient change conditions on causal influences can be used to recover the latent variables and causal graph up to certain indeterminacies. Our framework is connected with previous ones in a spectrum of related studies \citep{varici2023score, ahuja2023interventional, buchholz2022function,squires2023linear,brehmer2022weakly,von2023nonparametric,brehmer2022weakly,von2023nonparametric,zheng2023generalizing,zhang2023identifiability}. For instance, the connection between conditional independence and cross-derivatives of the log density in both linear and nonlinear cases means our theorems directly apply to linear SEMs. Furthermore, our results do not require the mixing function to be sufficiently nonlinear, allowing them to encompass linear mixing processes as well.

At the same time, we may be able to leverage possible parametric constraints on the data generating process (or functions) or specific types of interventions. For instance, if we know that the changes happen to the linear causal mechanisms with Gaussian noises, this constraint can readily help reduce the search space and improve the identifiability. Moreover, since we only require the changing distribution, any type of interventions will be covered since any change to the conditional distribution is allowed. Given the additional information illustrated by experimental interventions (e.g., single-node interventions), alternative identifiability that might be particularly useful in certain tasks can be established. We hope this work can provide a helpful, bigger picture of causal representation learning in the general setting and further illustrates the necessity and connections of the different assumptions formulated in this line of works.

\section{Change Encoding Network for Representation Learning}\label{sec:method}
Thanks to the identifiability result, we now present two different practical implementations to recover the latent variables and their causal relations from observations from multiple domains. 
We build our method on the variational autoencoder (VAE) framework and can be easily extended to other models, such as normalizing flows. 

We learn a
deep latent generative model (decoder) $p(X|Z;\hat{\theta}^{(u)})$ and a variational approximation (encoder) $q(Z|X, u)$ of its true posterior $p(Z|X;\theta^{(u)})$ since the true posterior is usually intractable. To learn the model,  we minimize the lower bound of the log-likelihood as 
\begin{align*}
    \log \ &p(X; \hat{\theta}^{(u)}) \\
    & = \log \int p(X|Z;\hat{\theta}^{(u)}) p(Z;\hat{\theta}^{(u)}) dZ \\ \nonumber
    &=\log \int \frac{q(Z|X, u)}{q(Z|X, u)}  p(X|Z;\hat{\theta}^{(u)}) p(Z;\hat{\theta}^{(u)}) dZ \\ \nonumber
    &\geq -\text{KL}(q(Z|X,u)|| p(Z;\hat{\theta}^{(u)})) +\mathbb{E}_q [\log p(X|Z;\hat{\theta}^{(u)})] \\ \nonumber
    & = -\mathcal{L}_\text{ELBO}
    \label{eq:elbo}
\end{align*}
For the posterior $q(Z|X, u)$, we assume that it is a multivariate Gaussian or a Laplacian distribution, where the mean and variance are generated by the neural network encoder. As for $q(X|Z)$, we assume that it is a multivariate Gaussian and the mean is the output of the decoder and the variance is a pre-defined value.

In practice, we can parameterize $p(X|Z;\hat{\theta}^{(u)})$ as the decoder which takes as input the latent representation $Z$ and $q(Z|X,u)$ as an encoder which outputs the mean and scale of the posterior distribution.  
An essential difference between VAE \citep{kingma2013auto} and iVAE \citep{khemakhem2020variational} is that our method allows the components of $Z$ to be causally dependent and we are able to learn the components and causal relationships. And 
the key is the prior distribution $P(Z;\hat{\theta}^{(u)})$. Now we present two different implementations to capture the changes with a properly defined prior distribution.

\subsection{Nonparametric Implementation of the Prior Distribution}
To recover the relationships and latent variables $Z$, we build the normalizing flow to mimic the inverse of the latent SEM $Z_i=f_i(\textrm{PA}(Z_i), \epsilon_i)$ in Eq. (\ref{eq:data_generating_process}). We first assume a causal ordering as $\hat{Z}_1, \dots, \hat{Z}_n$. 
Then, for each component $\hat{Z}_i$, we consider the previous components $\{\hat{Z}_1,\dots, \hat{Z}_{i-1} \}$ as potential parents of $\hat{Z}_i$ and we can select the true parents with the adjacency matrix $\hat{A}$, where $\hat{A}_{i,j}$ denotes that component $\hat{Z}_j$ contributes in the generation of $\hat{Z}_i$. If $\hat{A}_{i,j}=0$, it means that $\hat{Z}_j$ will not contribute to the generation of $\hat{Z}_i$. 
 Since $\theta^{(u)}$ governs the changes across domains, we use the observed domain index $u$ to discover the changes.
Then, we use the selected parents $\{\hat{A}_{i,1} \hat{Z}_1,\dots, \hat{A}_{i, i-1} \hat{Z}_{i-1}\}$ and the domain label $u$ to generate parameters of normalizing flow and apply the flow transformation on $\hat{Z}_i$ to turn it into $\hat{\epsilon}_i$. Specifically, we have 
\begin{align*}
    \hat{\epsilon}_i, \log\text{det}_i = \text{Flow}(\hat{Z}_i; \text{NN}(\{\hat{A}_{i,j}\hat{Z}_j\}_{j=1}^{i-1}, u)),
\end{align*}
where $\log\text{det}_i$ is the log determinant of the conditional flow transformation on $\hat{Z}_i$ and NN represents a neural network.

To compute the prior distribution, we make an assumption on the noise term $\epsilon$ that it follows an independent prior distribution $p(\epsilon)$, such as a standard isotropic Gaussian or a Laplacian. 
Then according to the change-of-variable formula, the prior distribution of the dependent latents can be written as 
\begin{align*}
    \log p(\hat{Z};\theta^{(u)}) = \sum_{i=1}^n(\log p(\hat{\epsilon}_i)+\log \text{det}_i).
\end{align*}
Intuitively, to minimize the KL divergence loss between $p(Z;\hat{\theta}^{(u)})$ and $q(Z|X, u)$, the network has to learn the correct structure and the underlying latent variables; otherwise, it can be difficult to transform the dependent latent variables $\hat{Z}$ to a factorized prior distribution, e.g., $\mathcal{N}(0, I)$.

\subsection{Parametric Implementation of the Prior Distribution}
We can make parametric assumption on the latent causal process and facilitate the learning of true causal structure and components. Here, we consider the linear SEM and more complex SEMs can be generalized. Specifically, we assume that the true generation process of the latent $Z$ is linear and only consists of scaling and shifting mechanisms:
\begin{align*}
Z=A(C^{(u)}Z)+S^{(u)}\epsilon + B^{(u)},
\end{align*}
where $A\in {[0,1]}^{n\times n}$ is a causal adjacency matrix which can be permuted to be strictly lower-triangular, $C^{(u)} \in \mathbb{R}^{n\times n}$ and $S^{(u)}\in \mathbb{R}^{n\times 1}$ are underlying domain-specific scaling matrix and vector for domain $u$, respectively, $B^{(u)} \in \mathbb{R}^{n\times 1}$ is the underlying domain-specific shift vector, and $\epsilon$ is the independent noise. 

To estimate the latent variables $Z$, the causal structure $A$, and capture the changes across domains, we introduce the learnable scaling $\hat{C} \in \mathbb{R}^{n\times n}, \hat{S} \in \mathbb{R}^{n \times 1}$and bias parameters $\hat{B} \in \mathbb{R}^{n\times 1}$ and pre-define a causal ordering as $\hat{Z}_1, \hat{Z}_2, \dots, \hat{Z}_n$.
Then we have the matrix form as 
\begin{align*}
    \hat{\epsilon} = (\hat{Z} - \hat{B}^{(u)} - \hat{A} \hat{C}^{(u)} \hat{Z})/\hat{S}^{(u)}.
\end{align*}
Note that the determinant of the strictly lower triangular matrix $\hat{C}$ is $0$.  Given a prior distribution of the noise term $p(\hat{\epsilon})$, and according to the change-of-variable rule, we then have the prior distribution for $\hat{Z}$ in parametric case as 
\begin{align*}
    \log p(\hat{Z};\hat{\theta}^{(u)}) = \sum_{i=1}^n (\log p(\hat{\epsilon}_i) -  \log |\hat{S}^{(u)}_i|).
\end{align*}

\subsection{Full Objective}
After we have properly defined the needed distributions $p(X|Z;\hat{\theta}^{(u)}), q(Z|X,u), p(Z;\hat{\theta}^{(u)})$, we can train our model to minimize the loss function $\mathcal{L}_\text{ELBO}$. However, without any further constraint, the powerful network may choose to use the fully connected causal graph during training. In other words, all lower-triangular elements of the estimated graph $\hat{A}$ is non-zero, which implies that each component $\hat{Z}_i$ is caused by all previous $i-1$ components. To exclude such unwanted solutions and encourage the model to learn the true causal structure among components of $Z$, we apply the $\ell_1$ regularization on $\hat{A}$, i.e.,
\begin{align*}
    \mathcal{L}_\text{sparsity} = \|\hat{A}\|_1.
\end{align*}
It is worth noting that the sparsity regularization term above is an approximation of the sparsity constraint on the edges of the estimated Markov network specified in \cref{theorem:identifiabiltiy_markov_network,theorem:identifiabiltiy_causal_variables}, since it is not straightforward to impose the latter constraint in a differentiable end-to-end training process. A more sophisticated alternative is to impose sparsity constraint on $(I-\hat{A})^T\Omega^{-1} (I-\hat{A})$ where $\Omega$ is a randomly sampled positive diagonal matrix. Note that this corresponds to the formula of precision matrix whose nonzero entries represent the moral graph under certain conditions~\citep{loh2014high} and we leave it for future investigation.

Finally, the full training objective is 
\begin{align*}
    \mathcal{L}_\text{full} = \mathcal{L}_\text{ELBO} + \mathcal{L}_\text{sparsity}.
\end{align*}
After the model converges, the output of the encoder $\hat{Z}$ is our recovered latents from the observations in multiple domains and 
the revealed causal structure is in $\hat{A}$ which encapsulates the causal relationships across the components.

\begin{figure*}[ht]
    \centering
   \begin{tabular}{cccc}
        \includegraphics[scale=0.13]{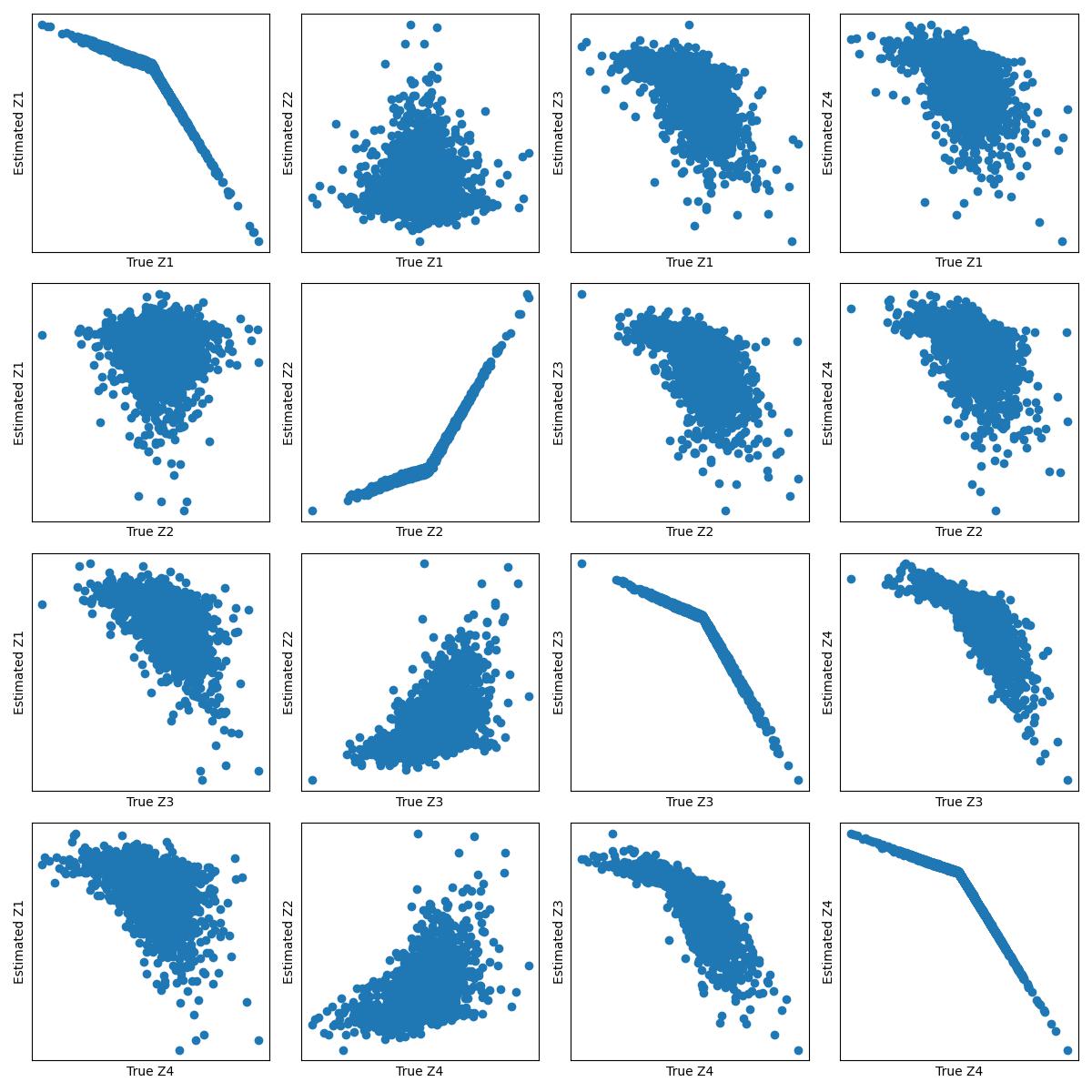}&
        \includegraphics[scale=0.13]{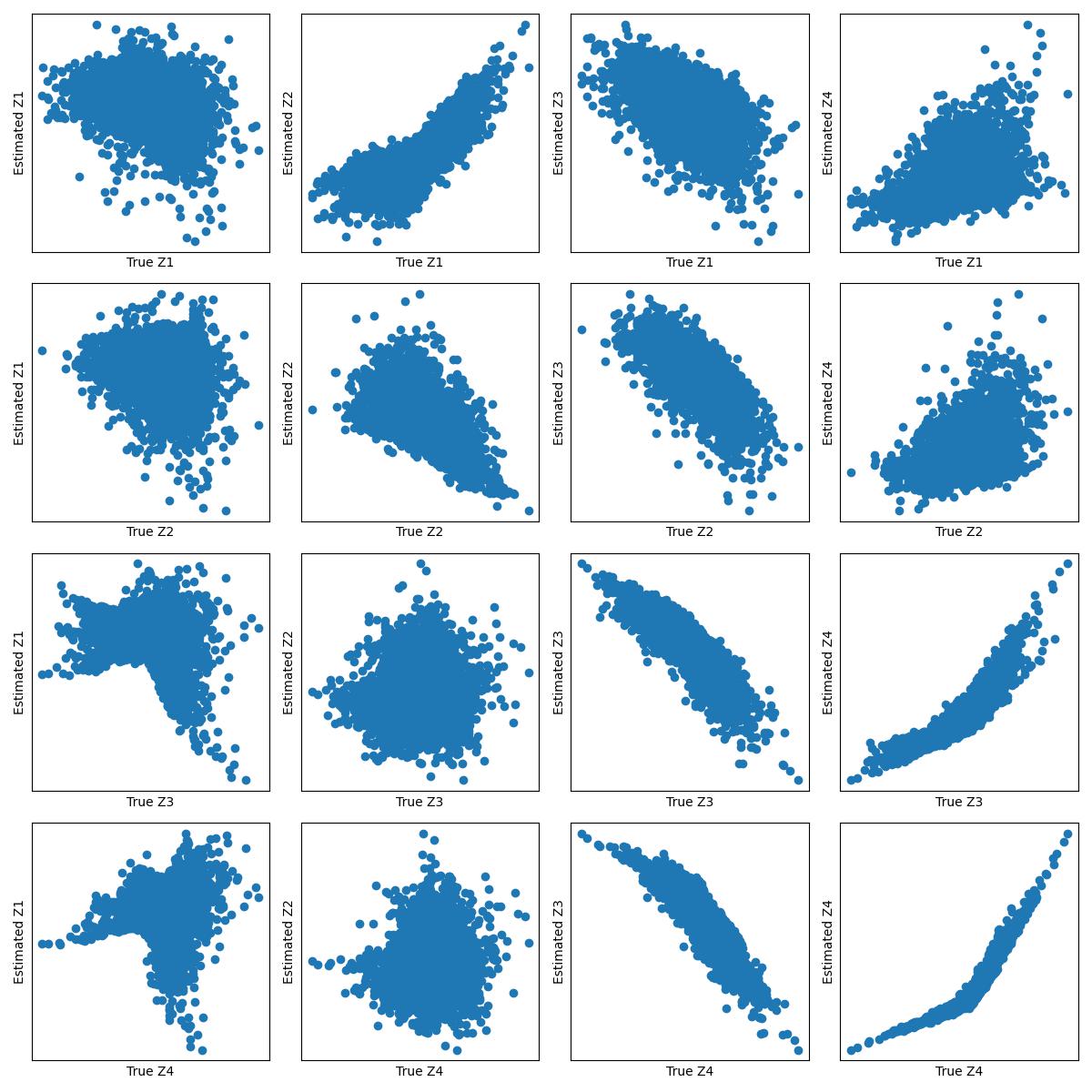} &
        \includegraphics[scale=0.13]{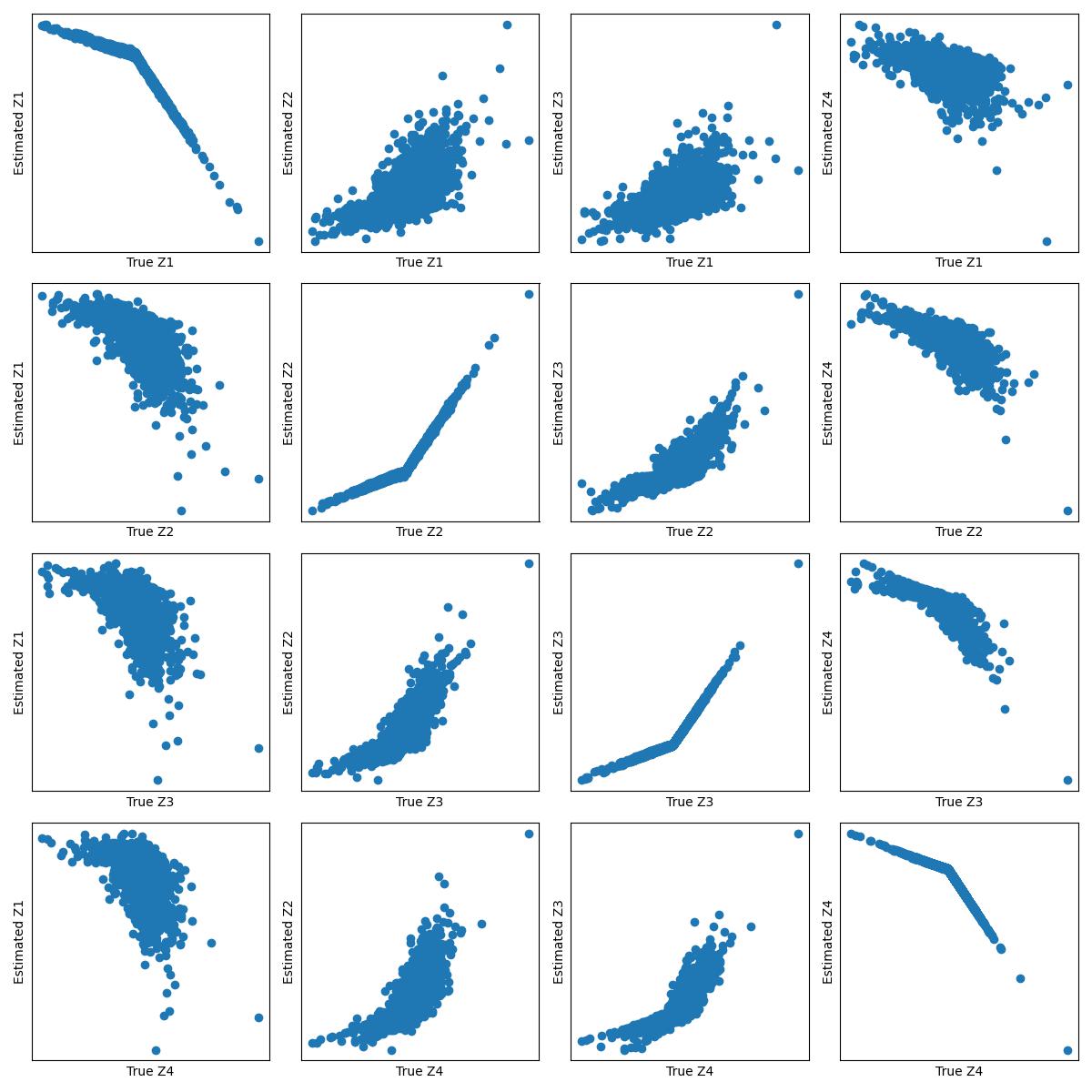} &
        \includegraphics[scale=0.13]{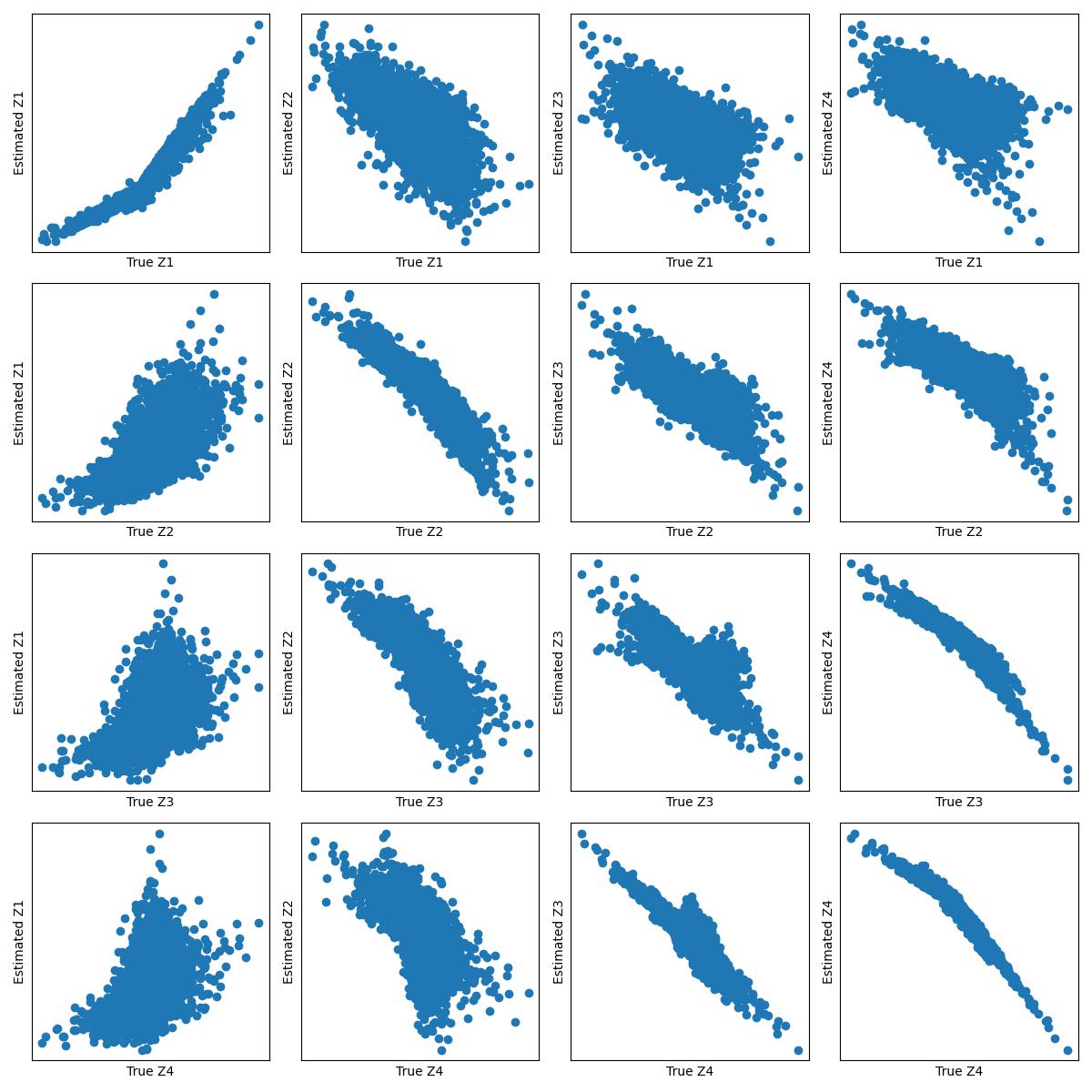}\\
   \end{tabular}
    \caption{Recovered latent variables v.s. the true latent variables with Non-Parametric Approach. (a) Y-structure with Laplace noise. (b) Y-structure with Gaussian noise. (c) Chain structure with Laplace noise. (d) Chain structure with Gaussian noise. In each sub-figure, $i$-th row and $j$-th column depcits the relationship between the estimated $\hat{Z}_i$ and the true components $Z_j$.}
    \label{fig:exp_nonlin}
\end{figure*}

\begin{figure*}[ht]
    \centering
   \begin{tabular}{cccc}
        \includegraphics[scale=0.13]{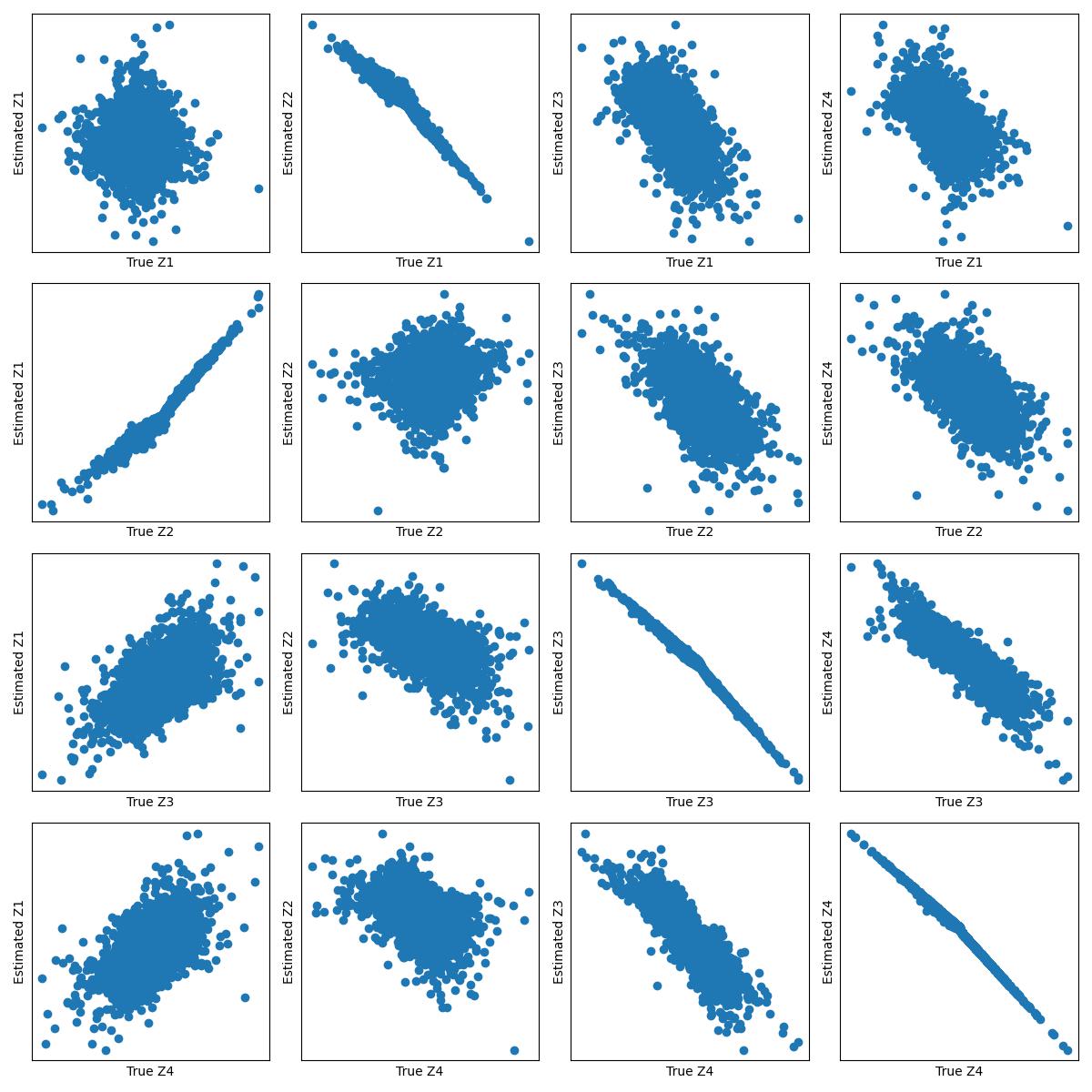}&
        \includegraphics[scale=0.13]{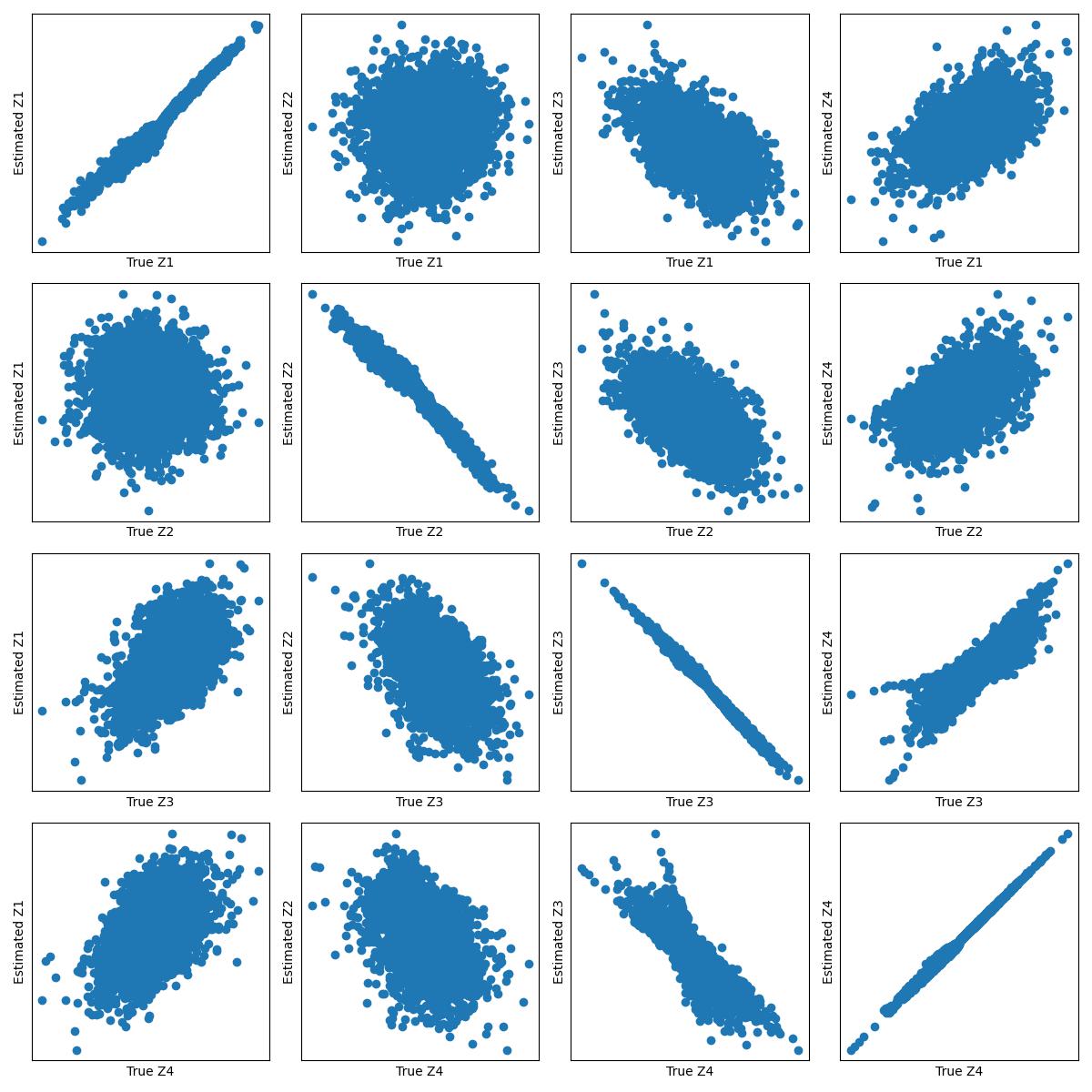} &
        \includegraphics[scale=0.13]{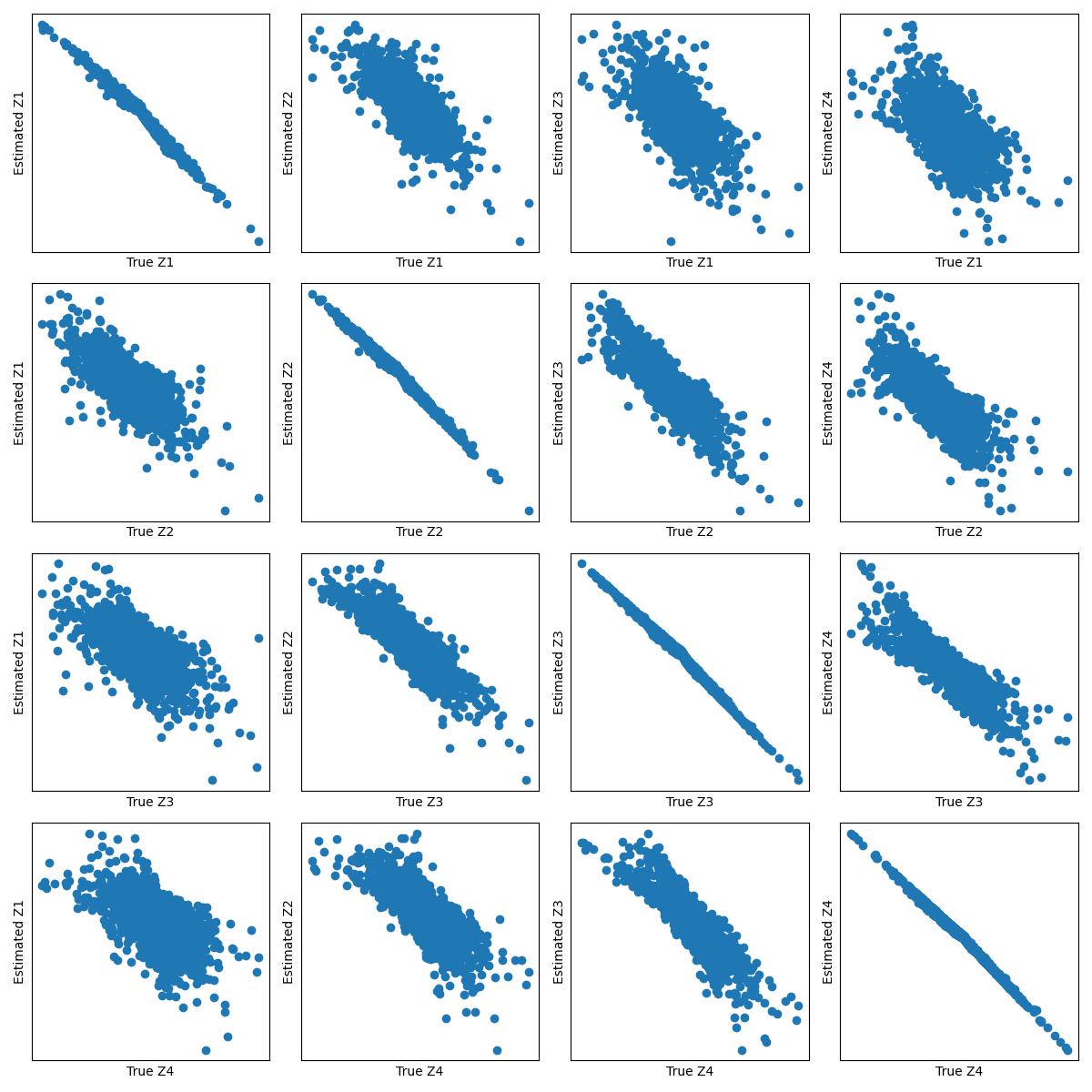} &
        \includegraphics[scale=0.13]{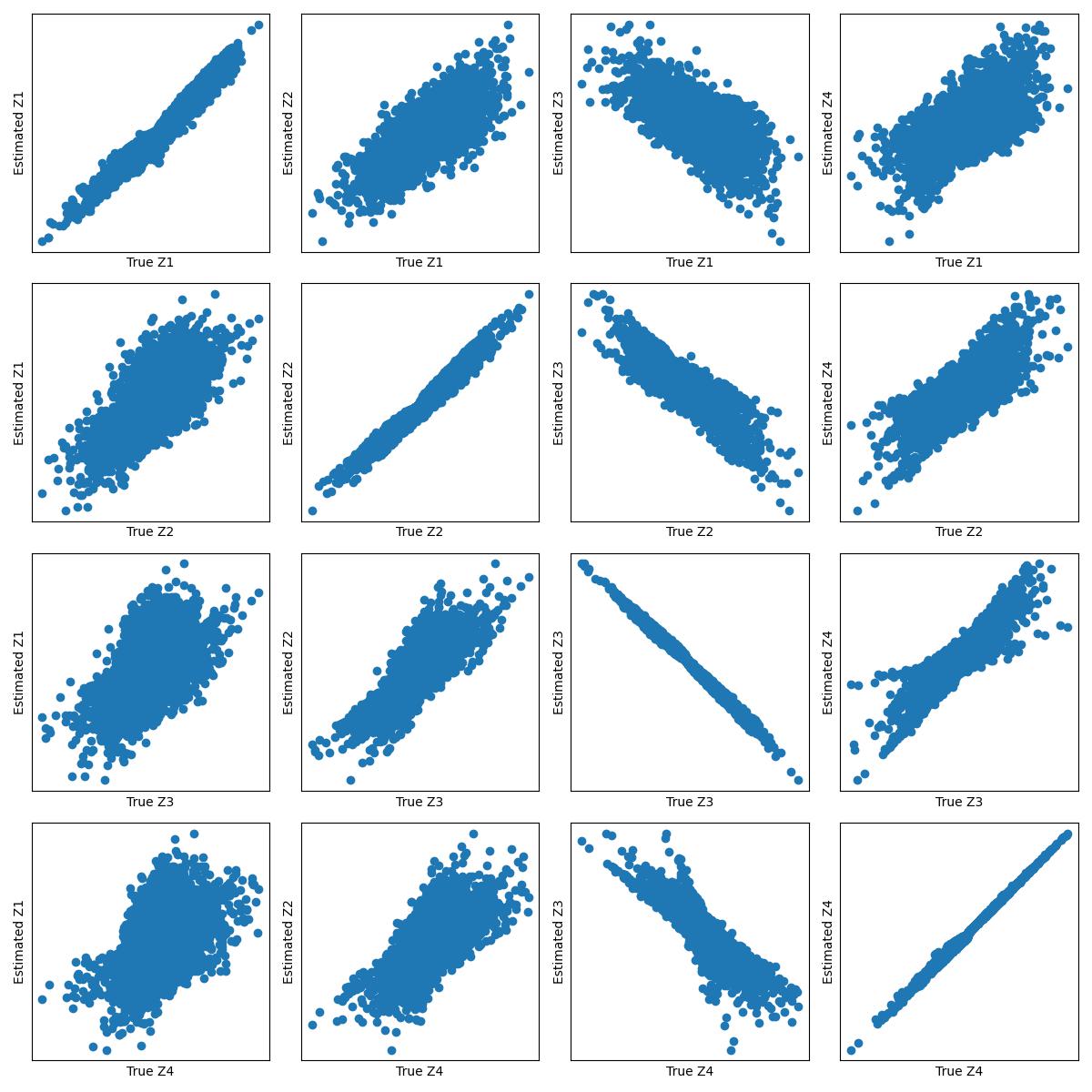}\\
   \end{tabular}
    \caption{Recovered latent variables v.s. the true latent variables with Linear Parameterization Approach. The $X$-axis denotes the components of true latent variables $Z$ and the $Y$-axis represent the components of estimated latent variables $\hat{Z}$.  (a) Y-structure with Laplace noise. (b) Y-structure with Gaussian noise. (c) Chain structure with Laplace noise. (d) Chain structure with Gaussian noise.\looseness=-1}
    \label{fig:exp_lin}
    \vspace{-0.3em}
\end{figure*}

\subsection{Simulations}

To verify our theory and the proposed implementations, we run experiments on the simulated data because the ground truth causal adjacency matrix and the latent variables across domains are available for simulated data. Consequently, we consider following common causal structures (i) Y-structure with 4 variables, $Z_1\rightarrow Z_3 \leftarrow Z_2, Z_3\rightarrow Z_4$ and (ii) chain structure $Z_1\rightarrow Z_2 \rightarrow Z_3 \rightarrow Z_4$. The noises are modulated with scaling random sampled from $\operatorname{Unif}[0.5, 2]$ and shifts are sampled from $\operatorname{Unif}[-2,2]$. The scaling on the $Z$ are also randomly sampled from $\operatorname{Unif}[0.5, 2]$. In other words, the changes are modular. After generating $Z$, we feed the latent variables into multilayer perceptron (MLP) with orthogonal weights and LeakyReLU activations for invertibility.  Specifically, we sample orthogonal matrix as the weights of the MLP layers. Since orthogonal matrix and LeakyReLU are invertible, the MLP function is also invertible.

We present the results in Figures \ref{fig:exp_nonlin} and \ref{fig:exp_lin}. Each sub-figure consist of $4\times 4$ panels and penal on $i$-th row and $j$-th column denotes the relationship between the estimated component $\hat{Z}_i$ with the true latent $Z_{j}$. We can see that under most cases, our model learns a strong one-to-one correspondence from the estimated components and the true components. For instance, the first column in Figure \ref{fig:exp_nonlin} show that $\hat{Z}_1$ is strongly correlated with the true components $Z_1$ while it is nearly independent from the true $Z_2$.

From the estimated $\hat{A}$, we find that our method is able to recover the true causal structure. For instance, on the Y-structure with $Z_1\rightarrow Z_3 \leftarrow Z_2$ and $Z_3\rightarrow Z_4$, our estimated model only keep the components $\hat{A}_{1,3}, \hat{A}_{2,3}, \hat{A}_{3,4}$ nonzero with the proposed sparsity regularization. The estimated causal graph is consistent with the true Y-structure causal graph. 
We can also see that the latent causal structure is also recovered from Figures \ref{fig:exp_lin} and \ref{fig:exp_nonlin}. We observe that the learned $\hat{Z}_1$ is strongly correlated with the true $Z_2$ and is independent from the true $Z_1$, but correlated with the $\hat{Z}_3$ and $\hat{Z}_4$. These results align well with the true causal graph since $Z_2$ is independent from $Z_1$ while is the cause of $Z_3$ and $Z_4$. \looseness=-1

The experiments support our theoretical result that the components and structure are identifiable up to certain indeterminacies. As for the results in Figure \ref{fig:exp_nonlin}, we observe that our non-parametric method is still able to recover the true latent variables with Laplace noise. 

\section{Related Work}

Causal representation learning aims to unearth causal latent variables and their relations from observed data. Despite its significance, the identifiability of the hidden generating process is known to be impossible without additional constraints, especially with only observational data. In the linear, non-Gaussian case, \citet{Silva06} recover the Markov equivalence class, provided that each observed variable has a unique latent causal parent; \citet{xie2020generalized, cai2019triad} estimate the latent variables and their relations assuming at least twice measured variables as latent ones, which has been further extended to learn the latent hierarchical structure \citep{xie2022identification}. Moreover, \citet{adams2021identification} provide theoretical results on the graphical conditions for identification. In the linear, Gaussian case, \citet{huang2022latent} leverage rank deficiency of the observed sub-covariance matrix to estimate the latent hierarchical structure, while \citet{dong2023versatile} further extend the rank constraint to accommodate flexibly related latent and observed variables. In the discrete case, \citet{kivva2021learning} identify the latent causal graph up to Markov equivalence by assuming a mixture model where the observed children sets of any pair of latent variables are different.

Given the challenge of identifiability on purely observational data, a different line of research leverage experiments by assuming the accessibility of various types of interventional data. Based on the single-node perfect intervention, \citet{squires2023linear} leverage single-node interventions for the identifiability of linear causal model and linear mixing function; \citep{varici2023score} incorporate for nonlinear causal model and linear mixing function; \citep{varici2023score, buchholz2023learning, jiang2023learning} provide the identifiability of the nonparametric causal model and linear mixing function; \citep{ahuja2023interventional} further generalize the result to nonparametric causal model and polynomial mixing functions with additional constraints on the latent support; and \citep{brehmer2022weakly, von2023nonparametric, jiang2023learning} explore the nonparametric settings for both the causal model and mixing function. In addition to the single-node perfect interventions, \citet{brehmer2022weakly} introduced counterfactual pre- and post-intervention views; \citet{von2023nonparametric} assume two distinct, paired interventions per node for multivariate causal models; \citet{zhang2023identifiability} explore soft interventions on polynomial mixing functions; and \citet{jiang2023learning} places specific structural restrictions on the latent causal graph.

Our study lies in the line of leveraging only observational data, and provides identifiability results in the general nonparametric settings on \textit{both} the latent causal model and mixing function.  Unlike prior works with observational data, we do not have any parametric assumptions or graphical restrictions; Compared to those relying on interventional data, our results naturally benefit from the heterogeneity of observational data (e.g., multi-domain data, nonstationary time series) and avoid additional experiments for interventions.

\section{Conclusion and Discussions}
We establish a set of new identifiability results to reveal latent causal variables and latent structures in the general nonparametric settings. Specifically, with sparsity regularization during estimation and sufficient changes in the causal influences, we demonstrate that the revealed latent variables and structures are related to the underlying causal model in a specific, nontrivial way.
In contrast to recent works on the recovery of latent causal variables and structures, our results rely on purely observational data without graphical or parametric constraints.
Our results offer insight into unveiling the latent causal process in one of the most universal settings. Experiments in various settings have been conducted to validate the theory. As future work, we will explore the scenario where only a subset of the causal relations change, which could be a challenge as well as a chance, and show up to what extent the underlying causal variables can be recovered with potentially weaker assumptions.
\looseness=-1

\vspace{-0.1em}
\section*{Acknowledgements}

The authors would like to thank the anonymous reviewers for helpful comments and suggestions. The authors would also like to acknowledge the support from NSF Grant 2229881, the National Institutes of Health (NIH) under Contract R01HL159805, and grants from Apple Inc., KDDI Research Inc., Quris AI, and Florin Court Capital.

\vspace{-0.1em}
\section*{Impact Statement}
This paper presents work whose goal is to advance the field of Machine Learning. There are many potential societal consequences of our work, none which we feel must be specifically highlighted here.

\bibliography{ref}
\bibliographystyle{icml2024}

\newpage
\appendix
\onecolumn
\begin{center}
{\LARGE \bf 
Supplementary Material}
\end{center}

\section{Proofs of Useful Lemmas}
\subsection{Proof of \cref{lemma:nonzero_diagonal_entries}}\label{app:proof:nonzero_diagonal_entries}
The following lemma is a rather standard result in linear algebra \citep{strang2006linear,strang2016introduction}, which has also been used in existing works in causal representation learning, such as \citet{lachapelle2021disentanglement}. We provide the proof here for completeness.
\begin{lemma}
\label{lemma:nonzero_diagonal_entries}
For any invertible matrix $A$, there exists a permutation of its columns such that the diagonal entries of the resulting matrix are nonzero.
\end{lemma}

\begin{proof}
Suppose by contradiction that there exists at least a zero diagonal entry for every column permutation. By Leibniz formula, we have
\[
    \det(A) = \sum_{\sigma \in \mathcal{S}_{n}} \left(\operatorname{sgn}(\sigma) \prod_{i=1}^{n} A_{i,\sigma(i)}\right),
\]
where $\mathcal{S}_{n}$ denotes the set of $n$-permutations. Since there exists a zero diagonal entry for every permutation, we have
\[
    \prod_{i=1}^{n} A_{i,\sigma(i)} = 0, \quad \forall \sigma \in \mathcal{S}_{n},
\]
which implies $\det(A) = 0$ and that matrix $A$ is not invertible. This is a contradiciton with the assumption that $A$ is invertible.
\end{proof}

\subsection{Proof of \cref{lemma:zero_submatrix}}
\begin{lemma}\label{lemma:zero_submatrix}
Suppose matrix $A\in\mathbb{R}^{n\times n}$ contains a zero submatrix of order $(i+1)\times (n-i)$ for some $i \in [n-1]$. Then, matrix $A$ is not invertible.
\end{lemma}
\begin{proof}
By the given condition, there exist $n-i$ columns in matrix $A$ of which $i+1$ rows are zero, i.e., at most $n-i-1$ rows are not zero. This implies that the $n-i$ column vectors span a subspace of dimension less than $n-i$, which thus are linearly dependent. Therefore, matrix $A$ cannot be invertible.
\end{proof}

\subsection{Proofs of \cref{lemma:subet_neighbors,lemma:intimate_neighbors}}
We provide the following lemmas that will be used to prove \cref{theorem:identifiabiltiy_causal_variables}.
\begin{lemma}\label{lemma:subet_neighbors}
Given Markov network $\mathcal{M}_Z$ over variables $Z$,  let $N_{Z_i}$ and $\Psi_{Z_i}$ be the neighbors and intimate neighbors of $Z_i$ in $\mathcal{M}_Z$, respectively. Then, for each $i\neq j$,  we have $Z_i\in\Psi_{Z_j}$ if and only if $\{Z_j\}\cup N_{Z_j}\subseteq \{Z_i\}\cup N_{Z_i}$.
\end{lemma}
\begin{proof}
We prove both directions as follows.

\textbf{Sufficient condition.} \ \
We proceed by contraposition. Suppose $Z_i\not\in\Psi_{Z_j}$. We consider the following two cases:
\begin{itemize}
\item Suppose $Z_i\not\in N_{Z_j}$. This implies $Z_j\not\in N_{Z_i}$ and thus $\{Z_j\}\not\subseteq \{Z_i\}\cup N_{Z_i}$. Therefore, $\{Z_j\}\cup N_{Z_j}\not\subseteq \{Z_i\}\cup N_{Z_i}$.
\item Suppose $Z_i\in N_{Z_j}$ and that there exists $Z_k\in N_{Z_j}$ such that $Z_k\not\in N_{Z_i}$ and $Z_k\neq Z_i$. Clearly, we have $Z_k\in \{Z_j\}\cup N_{Z_j}$ and $Z_k\not\in \{Z_i\}\cup N_{Z_i}$, which implies $\{Z_j\}\cup N_{Z_j}\not\subseteq \{Z_i\}\cup N_{Z_i}$.
\end{itemize}
Therefore, we have shown that if $Z_i\not\in\Psi_{Z_j}$, then $\{Z_j\}\cup N_{Z_j}\not\subseteq \{Z_i\}\cup N_{Z_i}$.

\textbf{Necessary condition.} \ \
For each $Z_k\in \{Z_j\}\cup N_{Z_j}$, we aim to show
\begin{equation}\label{eq:proof:subset_neighbors}
Z_k\in \{Z_i\}\cup N_{Z_i}.
\end{equation}
We consider the following two cases:
\begin{itemize}
\item Suppose $Z_k=Z_j$. Since $Z_i\in\Psi_{Z_j}$, we have $Z_j\in N_{Z_i}$, which implies $Z_k=Z_j\in N_{Z_i}\subsetneq \{Z_i\}\cup N_{Z_i}$.
\item Suppose $Z_k\in N_{Z_j}$. Clearly, Eq.~\eqref{eq:proof:subset_neighbors} is true if $Z_k=Z_i$. Now suppose $Z_k\in N_{Z_j}\setminus\{Z_i\}$. Since $Z_i\in\Psi_{Z_j}$, we have $Z_k\in N_{Z_i}$, which implies Eq.~\eqref{eq:proof:subset_neighbors}.
\end{itemize}
In each of the cases above, Eq.~\eqref{eq:proof:subset_neighbors} is true.
\end{proof}
\begin{lemma}\label{lemma:intimate_neighbors}
Given Markov network $\mathcal{M}_Z$ over variables $Z$,  let $N_{Z_i}$ and $\Psi_{Z_i}$ be the neighbors and intimate neighbors of $Z_i$ in $\mathcal{M}_Z$, respectively. Suppose $i\neq j$ and $\{Z_i\}\cup N_{Z_i}= \{Z_j\}\cup N_{Z_j}$. Then, we have  $\{Z_i\}\cup \Psi_{Z_i}= \{Z_j\}\cup \Psi_{Z_j}$.
\end{lemma}
\begin{proof}
By the given condition and \cref{lemma:subet_neighbors}, we have $Z_i\in \Psi_{Z_j}$ and $Z_j\in \Psi_{Z_i}$. We first prove $\{Z_j\}\cup \Psi_{Z_j}\subseteq\{Z_i\}\cup \Psi_{Z_i}$. That is, for each $Z_k\in \{Z_j\}\cup \Psi_{Z_j}$, we aim to show
\begin{equation}\label{eq:proof:intimate_neighbors}
Z_k\in \{Z_i\}\cup \Psi_{Z_i}.
\end{equation}
We consider the following two cases:
\begin{itemize}
\item Suppose $Z_k=Z_j$. Since $Z_j\in \Psi_{Z_i}$, we have $Z_k=Z_j\in \Psi_{Z_i}\subsetneq \{Z_i\}\cup \Psi_{Z_i}$.
\item Suppose $Z_k\in \Psi_{Z_j}$. By \cref{lemma:subet_neighbors}, we have $\{Z_j\}\cup N_{Z_j}\subseteq \{Z_k\}\cup N_{Z_k}$, which implies $\{Z_i\}\cup N_{Z_i}\subseteq \{Z_k\}\cup N_{Z_k}$ by the given condition. Applying \cref{lemma:subet_neighbors} again indicates $Z_k\in \Psi_{Z_i}\subsetneq \{Z_i\}\cup \Psi_{Z_i}$.
\end{itemize}
In each of the cases above, Eq.~\eqref{eq:proof:intimate_neighbors} is true, which indicates $\{Z_j\}\cup \Psi_{Z_j}\subseteq\{Z_i\}\cup \Psi_{Z_i}$. Similar reasoning can be straightforwardly applied to prove $\{Z_i\}\cup \Psi_{Z_i}\subseteq\{Z_j\}\cup \Psi_{Z_j}$. Therefore, we have $\{Z_i\}\cup \Psi_{Z_i}= \{Z_j\}\cup \Psi_{Z_j}$.
\end{proof}

\section{Proof of \cref{proposition:relation_markov_network}}\label{sec:proof:relation_markov_network}

\begin{repproposition}{proposition:relation_markov_network}
Let the observations be sampled from the data generating process in Eq. (\ref{eq:data_generating_process}), and $\mathcal{M}_Z$ be the Markov network over $Z$. Suppose the following assumptions hold:
\begin{itemize}
\item A1 (Smooth and positive density): The probability density function of latent causal variables, i.e., $p_Z$, is twice continuously differentiable and positive in $\mathbb{R}^n$.
\item A2 (Sufficient changes): For each value of $Z$, there exist $2n+|\mathcal{M}_Z|+1$ values of $\theta$, i.e., $\theta^{(u)}$ with $u=0,\dots,2n+|\mathcal{M}_Z|$, such that the vectors $w(Z, {u})-w(z,0)$ with $u=1,\dots,2n+|\mathcal{M}_Z|$ are linearly independent, where vector $w(Z, {u})$ is defined as follows:
\[
w(Z, {u})=\left(\frac{\partial \log p(Z;\theta^{(u)})}{\partial  Z_i}\right)_{i\in[n]}\oplus\left(\frac{\partial^2 \log p(Z;\theta^{(u)})}{\partial  Z_i^2}\right)_{i\in[n]}\oplus \left(\frac{\partial^2 \log p(Z;\theta^{(u)})}{\partial  Z_i \partial  Z_j}\right)_{\{Z_i,Z_j\}\in\mathcal{E}(\mathcal{M}_Z),\,i<j}.
\]
\end{itemize}
Suppose that we learn $(\hat{g}, \hat{f},p_{\hat{Z}},\hat{\Theta})$ to achieve Eq. (\ref{eq:matched_distribution}). Then, for every pair of estimated latent variables $\hat{Z}_k$ and $\hat{Z}_l$ that are {\bf not adjacent in the Markov network} $\mathcal{M}_{\hat{Z}}$ over $\hat{Z}$, we have the following statements:
\begin{enumerate}[label=(\alph*)]
\item For each true latent causal variable $Z_i$, we have
\[\frac{\partial  Z_i}{\partial \hat{Z}_k}\frac{\partial  Z_i}{\partial \hat{Z}_l}=0.\]
\item For each pair of true latent causal variables $Z_i$ and $Z_j$ that are adjacent in the Markov network $\mathcal{M}_Z$, we have 
\[\frac{\partial  Z_i}{\partial \hat{Z}_k}\frac{\partial  Z_j}{\partial \hat{Z}_l}=0.\]
\end{enumerate}
\end{repproposition}

\begin{proof}
Denote by $\operatorname{vol} A$ the volume of matrix $A$, which is the product of its singular values. Note that $\operatorname{vol} A=\sqrt{\det AA^T}$ when $A$ is of full row rank. In the change-of-variable formula, when the Jacobian is a rectangular matrix, the absolute determinant of the Jacobian can be replaced with the matrix volume~\citep{benisrael1999change,gemici2016normalizing,khemakhem2020variational}.\looseness=-1

Since $X=g(Z)$ and $\hat{X}=\hat{g}(\hat{Z})$, by  Eq.~(\ref{eq:matched_distribution}) and the change-of-variable formula, we have
\[
p_{\hat{X}}=p_X\;\implies\; p_{\hat{g}(\hat{Z})}=p_{g(Z)} \;\implies\; p_{g^{-1}\circ\hat{g}(\hat{Z})}\operatorname{vol} J_{g^{-1}}=p_{Z}\operatorname{vol} J_{g^{-1}}\;\implies\; p_{v(\hat{Z})}=p_{Z},
\]
where $J_{g^{-1}}$ is the Jacobian matrix of $g^{-1}$ and $v\coloneqq  g^{-1}\circ\hat{g}$ is a composition of diffeomorphisms (and hence also a diffeomorphism). Let $J_v$ be the Jacobian matrix of $v$. The change-of-variable formula implies
\begin{flalign}
p(\hat{Z};\hat{\theta})|\det J_{v^{-1}}| &= p(Z;\theta)\nonumber \\
\log p(\hat{Z};\hat{\theta}) &= \log p(Z;\theta) + \log|\det J_v|.\label{Eq:ZtoZtile}
\end{flalign}

Suppose $\hat{Z}_k$ and $\hat{Z}_l$ are conditionally independent given $\hat{Z}_{[n]\setminus\{k,l\}}$ i.e., they are not adjacent in the Markov network over $\hat{Z}$. For each $\hat{\theta}$, by \citet{lin1997factorizing}, we have
\begin{equation}\label{eq:cross_de_proof}\frac{\partial^2\log p(\hat{Z};\hat{\theta})}{\partial \hat{Z}_k \partial \hat{Z}_l} = 0.
\end{equation}
To see what it implies, we find the first-order derivative of Eq. \eqref{Eq:ZtoZtile}:
\[
\frac{\partial\log p(\hat{Z};\hat{\theta})}{\partial \hat{Z}_k} = \sum_{i=1}^n \frac{\partial \log p({Z};\theta)}{\partial  Z_i}\frac{\partial Z_i}{\partial \hat{Z}_k} + \frac{\partial\log|\det J_v|}{\partial \hat{Z}_k}.
\]
Let
\[\eta(\theta) \coloneqq \log p(Z;\theta),\quad \eta'_i(\theta) \coloneqq \frac{\partial \log p(Z;\theta)}{\partial  Z_i},\quad \eta''_{ij}(\theta) \coloneqq \frac{\partial^2 \log p({Z};\theta)}{\partial Z_i \partial Z_j},\quad h'_{i,l} \coloneqq \frac{\partial  Z_i}{\partial \hat{Z}_l},\quad \text{and}\quad h''_{i,kl} \coloneqq \frac{\partial^2 Z_i}{\partial \hat{Z}_k \partial \hat{Z}_l}.\]
We then derive the second-order derivative w.r.t. $\hat{Z}_k$ and $\hat{Z}_l$ and apply Eq. \eqref{eq:cross_de_proof}:
\begin{flalign} \nonumber
0 = &\sum_{j=1}^n \sum_{i=1}^n \frac{\partial^2 \log p({Z};\theta)}{\partial Z_i \partial Z_j} \frac{\partial  Z_j}{\partial \hat{Z}_l} \frac{\partial Z_i}{\partial \hat{Z}_k} + 
\sum_{i=1}^n \frac{\partial \log p({Z};\theta)}{\partial  Z_i}\frac{\partial^2 Z_i}{\partial \hat{Z}_k \partial \hat{Z}_l} 
+ \frac{\partial^2\log|\det J_v|}{\partial \hat{Z}_k \partial \hat{Z}_l} \\  \nonumber
=& \sum_{i=1}^n \frac{\partial^2 \log p({Z};\theta)}{\partial Z_i^2} \frac{\partial  Z_i}{\partial \hat{Z}_l} \frac{\partial Z_i}{\partial \hat{Z}_k} +  \sum_{j=1}^n \sum_{\substack{i:\{Z_j,Z_i\}\in \mathcal{E}(\mathcal{M}_Z)}} \frac{\partial^2 \log p({Z};\theta)}{\partial Z_i \partial Z_j} \frac{\partial  Z_j}{\partial \hat{Z}_l} \frac{\partial Z_i}{\partial \hat{Z}_k} \\
& \qquad +\sum_{i=1}^n \frac{\partial \log p({Z};\theta)}{\partial  Z_i}\frac{\partial^2 Z_i}{\partial \hat{Z}_k \partial \hat{Z}_l} + \frac{\partial^2\log|\det J_v|}{\partial \hat{Z}_k \partial \hat{Z}_l} \\ \label{Eq:2nd}
=& \sum_{i=1}^n \eta''_{ii}(\theta) h'_{i,l} h'_{i,k} + \sum_{j=1}^n \sum_{i:\{Z_j,Z_i\}\in \mathcal{E}(\mathcal{M}_Z)} \eta''_{ij}(\theta) h'_{j,l} h'_{i,k} + \sum_{i=1}^n \eta'_i(\theta) h''_{i,kl} + \frac{\partial^2\log|\det J_v|}{\partial \hat{Z}_k \partial \hat{Z}_l}.
\end{flalign}
Recall that $\mathcal{E}(\mathcal{M}_Z)$ denotes the set of edges in the Markov network over $Z$. In the equation above, we made use of the fact that if $Z_i$ and $Z_j$ are not adjacent in the Markov network, then $\frac{\partial^2 \log p(Z;\theta)}{\partial Z_i \partial Z_j} = 0$ by \citet{lin1997factorizing}.

By Assumption A2, consider the $2n+|\mathcal{M}_Z|+1$ values of $\theta$, i.e., $\theta^{(u)}$ with $u=0,\dots,2n+|\mathcal{M}_Z|$, such that Eq. (\ref{Eq:2nd}) hold. Then, we have $2n+|\mathcal{M}_Z|+1$ such equations. Subtracting each equation corresponding to $\theta^{(u)}, u=1,\dots,2n+|\mathcal{M}_Z|$ with the equation corresponding to $\theta^{(0)}$ results in $2n+|\mathcal{M}_Z|$ equations:
\begin{equation}\label{eq:initial_diff}
0=\sum_{i=1}^n (\eta''_{ii}(\theta^{(u)}) - \eta''_{ii}(\theta^{(0)})) h'_{i,l} h'_{i,k} + \sum_{j=1}^n \sum_{i:\{Z_j,Z_i\}\in \mathcal{E}(\mathcal{M}_Z)} (\eta''_{ij} (\theta^{(u)}) - \eta''_{ij} (\theta^{(0)})) h'_{j,l} h'_{i,k}  + \sum_{i=1}^n (\eta'_i(\theta^{(u)}) - \eta'_i(\theta^{(0)}) ) h''_{i,kl},
\end{equation}
where $u=1,\dots,2n+|\mathcal{M}_Z|$. Since $p_Z$ is twice continuously differentiable, we have
\[
\eta''_{ij} (\theta^{(u)}) - \eta''_{ij} (\theta^{(0)})=\eta''_{ji} (\theta^{(u)}) - \eta''_{ji} (\theta^{(0)}),
\]
and therefore Eq. \eqref{eq:initial_diff} can be written as
\vspace{-0.1em}
\begin{flalign*}
0=& \sum_{i=1}^n (\eta''_{ii}(\theta^{(u)}) - \eta''_{ii}(\theta^{(0)})) h'_{i,l} h'_{i,k}  + \sum_{\substack{i,j: \\i < j,\\ \{Z_i,Z_j\}\in \mathcal{E}(\mathcal{M}_Z)}} (\eta''_{ij} (\theta^{(u)}) - \eta''_{ij} (\theta^{(0)})) (h'_{j,l} h'_{i,k}+h'_{i,l} h'_{j,k}) \\
& \qquad + \sum_{i=1}^n (\eta'_i(\theta^{(u)}) - \eta'_i(\theta^{(0)}) ) h''_{i,kl}.
\end{flalign*}
Consider the vectors formed by collecting the corresponding coefficients in the equation above where $u=1,\dots,2n+|\mathcal{M}_Z|$. By Assumption A2, these $2n+|\mathcal{M}_Z|$  vectors are linearly independent. Thus, for any $i$ and $j$ such that $\{Z_i,Z_j\} \in \mathcal{E}(\mathcal{M}_Z)$, we have the following equations:
\vspace{-0.1em}
\begin{flalign}
h'_{i,k} h'_{i,l} &= 0, \label{eq:c1} \\
h'_{i,k} h'_{j,l}+h'_{j,k} h'_{i,l} & = 0, \label{eq:c2} \\
h''_{i,kl} & = 0.\nonumber
\end{flalign}
It remains to show $h'_{i,k} h'_{j,l}=0$. Suppose by contradiction that
\begin{equation}\label{eq:not_zero_contradiction}
h'_{i,k} h'_{j,l}\neq 0,
\end{equation}
which implies $h'_{i,k}\neq 0$. By Eq.~\eqref{eq:c1}, we have $h'_{i,l}=0$, which, by plugging into Eq.~\eqref{eq:c2}, indicates $h'_{i,k} h'_{j,l}=0$. This is a contradiction with Eq.~\eqref{eq:not_zero_contradiction}. Thus, we must have $h'_{i,k} h'_{j,l}=0$.
\end{proof}

\vspace{-0.5em}
\section{Proof of \cref{theorem:identifiabiltiy_markov_network}}\label{sec:proof:identifiabiltiy_markov_network}
\ThmIdentifiabilityMarkovNetwork*

\vspace{-0.55em}
\begin{proof}
Let $v\coloneqq  g^{-1}\circ\hat{g}$, i.e., $Z=v(\hat{Z})$. Note that $v$ is a composition of diffeomorphisms, and hence also a diffeomorphism. Consider a specific value of $\hat{Z}$, say $\hat{z}$. Since $v$ is diffeomorphism, by \cref{lemma:nonzero_diagonal_entries}, there exists a permutation $\pi$ such that the diagonal entries of the corresponding Jacobian matrix (whose columns are permuted according to $\pi$) evaluated at $\hat{Z}=\hat{z}$ are nonzero, i.e.,
\begin{equation} \label{EqC4}
\frac{\partial  Z_i}{\partial \hat{Z}_{\pi(i)}}\bigg|_{\hat{Z}=\hat{z}} \neq 0, \quad i=1,\dots,n.
\end{equation}
Suppose that $Z_i$ and $Z_j$ are adjacent in the Markov network $\mathcal{M}_Z$ over $Z$, but $\hat{Z}_{\pi(i)}$ and $\hat{Z}_{\pi(j)}$ are not adjacent in the Markov network $\mathcal{M}_{\hat{Z}}$ over $\hat{Z}$. By \cref{proposition:relation_markov_network}, we have
\[
\frac{\partial  Z_i}{\partial \hat{Z}_{\pi(i)}}\bigg|_{\hat{Z}=\hat{z}} \frac{\partial  Z_j}{\partial \hat{Z}_{\pi(j)}}\bigg|_{\hat{Z}=\hat{z}} =0,
\]
which is clearly a contradiction with Eq. (\ref{EqC4}).

Thus, we have shown by contradiction the following lemma.
\begin{lemma}\label{lemma:super_graph}
If $Z_i$ and $Z_j$ are adjacent in the Markov network $\mathcal{M}_Z$ over $Z$, then $\hat{Z}_{\pi(i)}$ and $\hat{Z}_{\pi(j)}$ are adjacent in the Markov network $\mathcal{M}_{\hat{Z}}$ over $\hat{Z}$.
\end{lemma}
\vspace{-0.2em}
The lemma above indicates 
\begin{equation}\label{eq:denser_markov_network}
|\mathcal{M}_{\hat{Z}}|\geq |\mathcal{M}_Z|.
\end{equation}
Also, note that the true model $(g, f,p_{Z},\Theta)$ is one of the solutions that achieves Eq. (\ref{eq:matched_distribution}). Since the recovered latent Markov network $\mathcal{M}_{\hat{Z}}$ has the minimal number of edges among the solutions  that achieve Eq. \eqref{eq:matched_distribution}, we have $|\mathcal{M}_{\hat{Z}}|\leq |\mathcal{M}_Z|$, which, with Eq.~\eqref{eq:denser_markov_network}, implies $|\mathcal{M}_{\hat{Z}}|=|\mathcal{M}_Z|$.

By \cref{lemma:super_graph} and $|\mathcal{M}_{\hat{Z}}|=|\mathcal{M}_Z|$, we conclude that $Z_i$ and $Z_j$ are adjacent in $\mathcal{M}_Z$ if and only if $\hat{Z}_{\pi(i)}$ and $\hat{Z}_{\pi(j)}$ are adjacent in $\mathcal{M}_{\hat{Z}}$. That is, $\mathcal{M}_Z$ and $\mathcal{M}_{\hat{Z}}$ are isomorphic.
\end{proof}

\section{Proof of \cref{theorem:partial_disentanglement}}\label{sec:proof:partial_disentanglement}
\ThmPartialDisentanglement*
\begin{proof}
We first prove Statement (a). By \cref{proposition:relation_markov_network}, for every value of $Z$, we have
\[
\frac{\partial  Z_i}{\partial \hat{Z}_k}\frac{\partial  Z_i}{\partial \hat{Z}_l}=0.
\]
Therefore, it suffices to prove that if $\frac{\partial  Z_i}{\partial \hat{Z}_{k}}\neq 0$ for some value of $\hat{Z}$, then $\frac{\partial  Z_i}{\partial \hat{Z}_{l}}= 0$ for all values of $\hat{Z}$. That is, these nonzero entries cannot switch positions.

By \cref{theorem:identifiabiltiy_markov_network}, there exists a permutation $\pi$ of the estimated variables, denoted as $\hat{Z}_{\pi}$, such that the Markov network $\mathcal{M}_{\hat{Z}_{\pi}}$ is identical to $\mathcal{M}_Z$.\footnote{The Markov networks $\mathcal{M}_Z$ and $\mathcal{M}_{\hat{Z}_{\pi}}$ are identical in the sense that $Z_i$ and $Z_j$ are adjacent in $\mathcal{M}_Z$ if and only if $\hat{Z}_{\pi(i)}$ and $\hat{Z}_{\pi(j)}$ are adjacent in $\mathcal{M}_{\hat{Z}_{\pi}}$.} Let $\hat{Z}_{\pi(i)}$ and $\hat{Z}_{\pi(k)}$ be two estimated latent variables that are not adjacent in the Markov network $\mathcal{M}_{\hat{Z}_{\pi}}$. Now consider variable $Z_i$. Suppose by contradiction that the nonzero entries switch positions, i.e., there exist two values of $\hat{Z}$, say $\hat{z}^{(1)}$ and $\hat{z}^{(2)}$, such that
\begin{equation}\label{eq:support_switch_places_1}
\frac{\partial  Z_i}{\partial \hat{Z}_{\pi(i)}}\bigg|_{\hat{Z}=\hat{z}^{(1)}}\neq0
\end{equation}
and
\begin{equation}\label{eq:support_switch_places_2}
\frac{\partial  Z_i}{\partial \hat{Z}_{\pi(k)}}\bigg|_{\hat{Z}=\hat{z}^{(2)}}\neq0,
\end{equation}

Let $N_{Z_i}$ be a set containing the indices of the neighbors of $Z_i$ in $\mathcal{M}_Z$, and $N_{\hat{Z}_{\pi(i)}}$ be a set containing the indices of the neighbors of $Z_{\pi(i)}$ in $\mathcal{M}_{\hat{Z}_{\pi}}$. Similarly, let $S_{Z_i}$ be a set containing the indices of the variables that are not adjacent to $Z_i$ in $\mathcal{M}_Z$, and $S_{\hat{Z}_{\pi(i)}}$ be a set containing the indices of the variables that are not adjacent to of $Z_{\pi(i)}$ in $\mathcal{M}_{\hat{Z}_{\pi}}$. By definition, we have\looseness=-1
\begin{equation}\label{eq:proof_indices_neighbors}
N_{Z_i}\cup S_{Z_i}\cup \{i\}=[n],
\end{equation}
which are pairwise disjoint.

Since $\mathcal{M}_{\hat{Z}_{\pi}}$ and $\mathcal{M}_Z$ are identical, we have $N_{Z_i}=N_{\hat{Z}_{\pi(i)}}$ and $S_{Z_i}=S_{\hat{Z}_{\pi(i)}}$. Now define the following function
\[
\phi(\hat{Z})=\sum_{j\in N_{Z_i}\cup\{i\}}\left(\frac{\partial  Z_j}{\partial \hat{Z}_{\pi(i)}}\right)^2-\sum_{l\in S_{Z_i}}\sum_{j\in N_{Z_i}\cup \{i\}}\left(\frac{\partial  Z_j}{\partial \hat{Z}_{\pi(l)}}\right)^2.
\]
Plugging in $\hat{Z}=\hat{z}^{(1)}$, for $l\in S_{Z_i}=S_{\hat{Z}_{\pi(i)}}$ and $j\in N_{Z_i}\cup \{i\}$, \cref{proposition:relation_markov_network} implies
\[
\frac{\partial  Z_i}{\partial \hat{Z}_{\pi(i)}}\bigg|_{\hat{Z}=\hat{z}^{(1)}}\frac{\partial  Z_j}{\partial \hat{Z}_{\pi(l)}}\bigg|_{\hat{Z}=\hat{z}^{(1)}}=0,
\]
which, with Eq. \eqref{eq:support_switch_places_1}, indicates
\[
\frac{\partial  Z_j}{\partial \hat{Z}_{\pi(l)}}\bigg|_{\hat{Z}=\hat{z}^{(1)}}=0.
\]
Substituting the above equation and  Eq. \eqref{eq:support_switch_places_1} into function $\phi$,
we have
\begin{equation}\label{eq:f_larger_than_zero}
\phi(\hat{Z})|_{\hat{Z}=\hat{z}^{(1)}}=\sum_{j\in N_{Z_i}\cup\{i\}}\left(\frac{\partial  Z_j}{\partial \hat{Z}_{\pi(i)}}\bigg|_{\hat{Z}=\hat{z}^{(1)}}\right)^2\geq \left(\frac{\partial  Z_i}{\partial \hat{Z}_{\pi(i)}}\bigg|_{\hat{Z}=\hat{z}^{(1)}}\right)^2>0.
\end{equation}
Now plug in $\hat{Z}=\hat{z}^{(2)}$. For $j\in N_{Z_i}\cup \{i\}$, \cref{proposition:relation_markov_network} implies
\[
\frac{\partial  Z_j}{\partial \hat{Z}_{\pi(i)}}\bigg|_{\hat{Z}=\hat{z}^{(2)}}\frac{\partial  Z_i}{\partial \hat{Z}_{\pi(k)}}\bigg|_{\hat{Z}=\hat{z}^{(2)}}=0,
\]
which, with Eq. \eqref{eq:support_switch_places_2}, indicates
\[
\frac{\partial  Z_j}{\partial \hat{Z}_{\pi(i)}}\bigg|_{\hat{Z}=\hat{z}^{(2)}}=0.
\]
Substituting the above equation and  Eq. \eqref{eq:support_switch_places_2} into function $\phi$,
we have
\begin{equation}\label{eq:f_smaller_than_zero}
\phi(\hat{Z})|_{\hat{Z}=\hat{z}^{(2)}}=-\sum_{l\in S_{Z_i}}\sum_{j\in N_{Z_i}\cup \{i\}}\left(\frac{\partial  Z_j}{\partial \hat{Z}_{\pi(l)}}\bigg|_{\hat{Z}=\hat{z}^{(2)}}\right)^2\leq -\left(\frac{\partial  Z_i}{\partial \hat{Z}_{\pi(k)}}\bigg|_{\hat{Z}=\hat{z}^{(2)}}\right)^2<0.
\end{equation}
Since function $\phi$ is continuous (because all the partial derivatives involved are continuous) and its domain is a connected set, by applying Intermediate Value Theorem with Eqs. \eqref{eq:f_larger_than_zero} and \eqref{eq:f_smaller_than_zero}, there exists a value of $\hat{Z}$ in the domain, say $\hat{z}^{(3)}$, such that
\[
\phi(\hat{Z})|_{\hat{Z}=\hat{z}^{(3)}}=0,
\]
which, by plugging the definition of function $\phi$, implies
\[
\sum_{j\in N_{Z_i}\cup\{i\}}\left(\frac{\partial  Z_j}{\partial \hat{Z}_{\pi(i)}}\bigg|_{\hat{Z}=\hat{z}^{(3)}}\right)^2=\sum_{l\in S_{Z_i}}\sum_{j\in N_{Z_i}\cup \{i\}}\left(\frac{\partial  Z_j}{\partial \hat{Z}_{\pi(l)}}\bigg|_{\hat{Z}=\hat{z}^{(3)}}\right)^2.
\]
Note that if any of the terms in the summation on the left hand side (LHS) is nonzero, then, by \cref{proposition:relation_markov_network}, all terms in the summation on the right hand side (RHS) must be zero; in this case, LHS is nonzero but RHS equals zero, which is a contradiction. Similarly, if any of the terms in the summation on the RHS is nonzero, then, by \cref{proposition:relation_markov_network}, all terms in the summation on the LHS must be zero; in this case, RHS is nonzero but LHS equals zero, which is a contradiction. This implies that all terms in the summation on both LHS and RHS must be zero, i.e.,
\[
\frac{\partial  Z_j}{\partial \hat{Z}_{\pi(l)}}\bigg|_{\hat{Z}=\hat{z}^{(3)}}=0 \quad\textrm{for}\quad j\in N_{Z_i}\cup\{i\},l\in S_{Z_i}\cup\{i\}.
\]
Since $|S_{Z_i}\cup\{i\}|=n-|N_{Z_i}|$ by Eq. \eqref{eq:proof_indices_neighbors}, \cref{lemma:zero_submatrix} indicates that the matrix $\frac{\partial  Z}{\partial \hat{Z}_\pi}\big|_{\hat{Z}=\hat{z}^{(3)}}$ is not invertible. Thus, the (Jacobian) matrix $\frac{\partial  Z}{\partial \hat{Z}}\big|_{\hat{Z}=\hat{z}^{(3)}}$ is also not invertible, which is a contradiction because the mapping from $\hat{Z}$ to $Z$ is a diffeomorphism (specifically a a composition of diffeomorphisms).

Therefore, we have just proved Statement (a) by contradiction. Similar reasoning can be straightforwardly applied to prove Statement (b) and is omitted here.
\end{proof}

\section{Proof of \cref{theorem:identifiabiltiy_causal_variables}}\label{sec:proof:identifiabiltiy_causal_variables}
We first state the following lemma that is used to prove \cref{theorem:identifiabiltiy_causal_variables}. The proof is a straightforward consequence of Cayley–Hamilton theorem and is omitted here.
\begin{lemma}\label{lemma:matrix_inverse_as_powers}
Let $A$ be an $n\times n$ invertible matrix. Then, it can be expressed as a linear combination of the powers of $A$, i.e., 
\[
A^{-1}=\sum_{k=0}^{n-1} c_k A^k
\]
for some appropriate choice of coefficients $c_0,c_1,\dots,c_{n-1}$.
\end{lemma}
Now consider the Markov network $\mathcal{M}_Z$ over variables $Z$. With a slight abuse of notation, let $N_{Z_i}$ be the set of neighbors of $Z_i$ in $\mathcal{M}_Z$. The following result relates a matrix to its inverse, given that the matrix satisfies certain property defined by $\mathcal{M}_Z$.\looseness=-1
\begin{proposition}\label{prop:subset_support}
Consider Markov network $\mathcal{M}_{Z}$ over $Z$. Let $N_{Z_i}$ be the set of neighbors of $Z_i$ in $\mathcal{M}_Z$, and $A$ be an $n\times n$ invertible matrix. For each $i\neq j$ where $Z_j$ is not adjacent to some nodes in $\{Z_i\}\cup (N_{Z_i}\setminus \{Z_j\})$, suppose $A_{ij}=0$. Then, $A^{-1}_{ij}=0$.
\end{proposition}
\begin{proof}
By \cref{lemma:matrix_inverse_as_powers}, $A^{-1}$ can be expressed as linear combination of the powers of $A$. Therefore, it suffices to prove that each matrix power $A^k$ satisfies the following property: $A^{k}_{ij}=0$ for each $i\neq j$ where $Z_j$ is not adjacent to some nodes in $\{Z_i\}\cup (N_{Z_i}\setminus \{Z_j\})$. We proceed with mathematical induction on $k$. By definition, the property holds in the base case where $k=1$.

Now suppose that the property holds for $A^k$. We prove by contradiction that the property holds for $A^{k+1}$. Suppose by contradiction that $A^{k+1}_{ij}\neq 0$ for some $i\neq j$ where $Z_j$ is not adjacent to some nodes in $\{Z_i\}\cup (N_{Z_i}\setminus \{Z_j\})$. This implies that one of the following cases holds:
\vspace{-0.4em}
\begin{itemize}
\item Case (a): $Z_j$ is not adjacent to $Z_i$ in $\mathcal{M}_{Z}$.
\item Case (b): There exists $Z_l\in N_{Z_i}\setminus \{Z_j\}$ such that $Z_j$ and $Z_l$ are not adjacent in $\mathcal{M}_{Z}$.
\end{itemize}
Since $A^{k+1}_{ij}=\sum_{r=0}^{n}A^k_{ir}A_{rj}$, the assumption $A^{k+1}_{ij}\neq 0$ implies that there must exist $m$ such that $A^k_{im}A_{mj}\neq 0$, i.e., $A^k_{im}\neq 0$ and $A_{mj}\neq 0$. Since both $A^k$ and $A$ satisfy the property, this indicates (i) $Z_m$ is adjacent to $Z_i$ and all nodes in $N_{Z_i}\setminus\{Z_m\}$, and (ii) $Z_j$ is adjacent to $Z_m$ and all nodes in $N_{Z_m}\setminus\{Z_j\}$. We consider the following cases:
\vspace{-0.4em}
\begin{itemize}
\item Case of $m=l$: By (ii), $Z_j$ is adjacent to $Z_l$, which contradicts Case (b) above. Also, we know that $Z_l$ is adjacent to $Z_i$ by (i), which indicates that $Z_i$ is adjacent to $Z_j$, contradicting Case (a) above.
\item Case of $m\neq l$: By (i) and (ii), $Z_m$ is adjacent to $Z_i$ and $Z_j$ is adjacent to $Z_m$, implying that $Z_i$ and $Z_j$ are adjacent, which is contradictory with Case (a) above. Furthermore, since $Z_l$ is a neighbor of $Z_i$, we know that $Z_m$ and $Z_l$ are adjacent by (i). Also, by (ii), $Z_j$ is adjacent to $Z_l$, which contradicts Case (b) above.
\end{itemize}
\vspace{-0.2em}
In each of the cases above, there is a contradiction.
\end{proof}

Before proving \cref{theorem:identifiabiltiy_causal_variables}, we first establish a weaker form of the identifiability result below. This result may be of interest on its own, as it provides useful information for disentanglement.
\begin{proposition}\label{proposition:weaker_identifiabiltiy_causal_variables}
Let the observations be sampled from the data generating process in Eq. (\ref{eq:data_generating_process}). Suppose that Assumptions A1 and A2 from Theorem 1 hold. Suppose also that we learn $(\hat{g}, \hat{f},p_{\hat{Z}},\hat{\Theta})$ to achieve Eq. (\ref{eq:matched_distribution})  {with
the minimal number of edges of the Markov network $\mathcal{M}_{\hat{Z}}$ over $\hat{Z}$}. Then, there exists a permutation $\pi$ of the estimated latent variables, denoted as $\hat{Z}_{\pi}$, such that each $\hat{Z}_{\pi(i)}$ is solely a function of a subset of $\{Z_i\}\cup\Psi_{Z_i}$.
\end{proposition}

\begin{proof}
We first prove a simpler case: there exists a permutation $\pi$ of the estimated latent variables, denoted as $\hat{Z}_\pi$, such that $Z_i$ is solely a function of $\hat{Z}_{\pi(i)}$ and a subset of $\{\hat{Z}_{\pi(r)} \,|\, Z_r \in{\Psi}_{Z_i}\}$.

By \cref{theorem:identifiabiltiy_markov_network} and its proof, there exists a permutation $\pi$ of the estimated variables, denoted as $\hat{Z}_{\pi}$, such that the Markov network $\mathcal{M}_{\hat{Z}_{\pi}}$ over $\hat{Z}_{\pi}$ is identical to $\mathcal{M}_Z$,\footnote{The Markov networks $\mathcal{M}_Z$ and $\mathcal{M}_{\hat{Z}_{\pi}}$ are identical in the sense that $Z_i$ and $Z_j$ are adjacent in $\mathcal{M}_Z$ if and only if $\hat{Z}_{\pi(i)}$ and $\hat{Z}_{\pi(j)}$ are adjacent in $\mathcal{M}_{\hat{Z}_{\pi}}$.} and that
\[
\frac{\partial  Z_i}{\partial \hat{Z}_{\pi(i)}} \neq 0, \quad i=1,\dots,n.
\]
Clearly, each variable $Z_i$ is a function of $\hat{Z}_{\pi(i)}$.

We first show that if $Z_j$ is not adjacent to $Z_i$ in $\mathcal{M}_{Z}$, then $Z_i$ cannot be a function of $\hat{Z}_{\pi(j)}$. Since $Z_i$ and $Z_j$ are not adjacent in $\mathcal{M}_{Z}$, we know that $\hat{Z}_{\pi(i)}$ and $\hat{Z}_{\pi(j)}$ are not adjacent in $\mathcal{M}_{{\hat{Z}_{\pi}}}$. By \cref{theorem:partial_disentanglement}, $Z_i$ is a function of at most one of $\hat{Z}_{\pi(i)}$ and $\hat{Z}_{\pi(j)}$, which implies that $Z_i$ cannot be a function of $\hat{Z}_{\pi(j)}$, because we have shown that $Z_i$ is a function of $\hat{Z}_{\pi(i)}$.

To refine further, now suppose that $Z_j$ is adjacent to $Z_i$, but not adjacent to some $Z_k\in N_{Z_i}\setminus\{Z_j\}$. Since $\mathcal{M}_{Z}$ and $\mathcal{M}_{{\hat{Z}_{\pi}}}$ are identical, $\hat{Z}_{\pi(j)}$ is also not adjacent to $\hat{Z}_{\pi(k)}$ in $\mathcal{M}_{\hat{Z}_{\pi}}$. Since $Z_i$ and $Z_k$ are adjacent in $\mathcal{M}_{Z}$, 
by \cref{theorem:partial_disentanglement}, at most one of them is a function of $\hat{Z}_{{\pi(j)}}$ or $\hat{Z}_{\pi(k)}$. This implies that $Z_i$ cannot be a function of $\hat{Z}_{\pi(j)}$, because we have shown that $Z_k$ is a function of $\hat{Z}_{\pi(k)}$.

Therefore, we have just shown that $Z_i$ is solely a function of $\hat{Z}_{\pi(i)}$ and a subset of $\{\hat{Z}_{\pi(r)} \,|\, Z_r \in{\Psi}_{Z_i}\}$. Now consider variable $Z_l\not\in \{Z_i\}\cup{\Psi}_{Z_i}$. Since $Z_i$ is not a function of $\hat{Z}_{\pi(l)}$, we have
\[
\left(\frac{\partial Z}{\partial \hat{Z}_{\pi}}\right)_{il}=\frac{\partial Z_i}{\partial \hat{Z}_{\pi(l)}}=0.
\]
By applying Proposition \ref{prop:subset_support} with matrix $\frac{\partial Z}{\partial \hat{Z}_{\pi}}$, we have
\[
\left(\frac{\partial Z}{\partial \hat{Z}_{\pi}}\right)^{-1}_{il}=0,
\]
which, by Inverse Function Theorem, implies
\[
\frac{\partial \hat{Z}_{\pi(i)}}{\partial Z_l}=\left(\frac{\partial \hat{Z}_\pi}{\partial Z}\right)_{il}=\left(\frac{\partial Z}{\partial \hat{Z}_{\pi}}\right)^{-1}_{il}=0.
\]
Since the above equation holds for all values of $Z$, we conclude that $\hat{Z}_{\pi(i)}$ cannot be a function of $Z_l$. 
\end{proof}

We are now ready to prove our main identifiability result of the latent causal variables. The only difference with \cref{proposition:weaker_identifiabiltiy_causal_variables} is that the following result shows that $\hat{Z}_{\pi(i)}$ must be a function of $Z_i$ for  permutation $\pi$.
\ThmIdentifiabilityCausalVariables*
\begin{proof}
By \cref{proposition:weaker_identifiabiltiy_causal_variables},  there exists a permutation $\sigma$ of the estimated latent variables, denoted as $\hat{Z}_{\sigma}$, such that each $\hat{Z}_{\sigma(i)}$ is solely a function of a subset of $\{Z_i\}\cup\Psi_{Z_i}$.

By \cref{lemma:nonzero_diagonal_entries}, there exists a permutation $\pi$ such that the diagonal entries of the Jacobian matrix $\frac{\partial \hat{Z}_\pi}{\partial Z}$ are nonzero. Considering the transformation from $\hat{Z}_{\sigma}$ to $\hat{Z}_{\pi}$, the variables are transformed according to $\pi\circ\sigma^{-1}$, which is also a permutation. Note that any permutation (on finitely many elements) can be decomposed into disjoint cyclic subpermutations \citep{ehrlich2013fundamental}. These subpermutations may include trivial ones, where a single element remains in its original position.

Consider variable $\hat{Z}_{\pi(i)}$. Suppose that it was involved in a trivial cyclic subpermutation (from $\hat{Z}_{\sigma}$ to $\hat{Z}_{\pi}$), i.e., $\pi(i)=\sigma(i)$. Since the diagonal entries of $\frac{\partial \hat{Z}_{\pi}}{\partial Z}$ are nonzero by definition, clearly $\hat{Z}_{\pi(i)}$ is a function of $Z_i$. Also, we have $\frac{\partial \hat{Z}_{\pi(i)}}{\partial Z}=\frac{\partial \hat{Z}_{\sigma(i)}}{\partial Z}$, which indicates that $\hat{Z}_{\pi(i)}$ is solely a function of $Z_i$ and a subset of $\Psi_{Z_i}$.

Now suppose that $\hat{Z}_{\pi(i)}$ was involved in a nontrivial cyclic subpermutation (from $\hat{Z}_{\sigma}$ to $\hat{Z}_{\pi}$). That is, there exists a sequence $j_1,\dots,j_k$ such that 
\begin{equation}\label{eq:proof_identifiability_permutation_indices}
\pi(j_1)=\sigma(i),\quad\pi(j_2)=\sigma(j_1),\quad \pi(j_3)=\sigma(j_2),\quad \dots,\quad \pi(j_k)=\sigma(j_{k-1}),\quad\pi(i)=\sigma(j_k).
\end{equation}
Since the diagonal entries of $\frac{\partial \hat{Z}_\pi}{\partial Z}$ are nonzero by definition, we have
\begin{equation}\label{eq:proof_identifiability_nonzero_derivative}
\frac{\partial\hat{Z}_{\sigma(i)}}{\partial Z_{j_1}}\neq 0,\quad \frac{\partial \hat{Z}_{\sigma(j_1)}}{\partial Z_{j_2}}\neq 0,\quad \frac{\partial \hat{Z}_{\sigma(j_2)}}{\partial Z_{j_3}}\neq 0,\quad\dots,\quad \frac{\partial \hat{Z}_{\sigma(j_{k-1})}}{\partial Z_{j_k}}\neq 0,\quad \frac{\partial \hat{Z}_{\sigma(j_k)}}{\partial Z_{i}}\neq 0,
\end{equation}
which indicate
\[
Z_{j_1}\in\Psi_{Z_i},\quad Z_{j_2}\in\Psi_{Z_{j_1}},\quad Z_{j_3}\in\Psi_{Z_{j_2}},\quad\dots,\quad Z_{j_k}\in\Psi_{Z_{j_{k-1}}},\quad Z_{i}\in\Psi_{Z_{j_k}}.
\]
By \cref{lemma:subet_neighbors} and the above equation, we have
\[
\{Z_i\}\cup N_{Z_i}
\,\,\subseteq\,\, \{Z_{j_1}\}\cup N_{Z_{j_1}}
\,\,\subseteq\,\, \{Z_{j_2}\}\cup N_{Z_{j_2}}
\,\,\subseteq\,\, \{Z_{j_3}\}\cup N_{Z_{j_3}}
\,\,\subseteq\,\, \dots
\,\,\subseteq\,\, \{Z_{j_k}\}\cup N_{Z_{j_k}}
\,\,\subseteq\,\, \{Z_i\}\cup N_{Z_i}
\]
and thus
\[
\{Z_i\}\cup N_{Z_i}
\,\,=\,\, \{Z_{j_1}\}\cup N_{Z_{j_1}}
\,\,=\,\, \{Z_{j_2}\}\cup N_{Z_{j_2}}
\,\,=\,\, \{Z_{j_3}\}\cup N_{Z_{j_3}}
\,\,=\,\, \dots
\,\,=\,\, \{Z_{j_k}\}\cup N_{Z_{j_k}}.
\]
Applying \cref{lemma:intimate_neighbors} with the above equation implies
\[
\{Z_i\}\cup \Psi_{Z_i}
\,\,=\,\, \{Z_{j_1}\}\cup \Psi_{Z_{j_1}}
\,\,=\,\, \{Z_{j_2}\}\cup \Psi_{Z_{j_2}}
\,\,=\,\, \{Z_{j_3}\}\cup \Psi_{Z_{j_3}}
\,\,=\,\, \dots
\,\,=\,\, \{Z_{j_k}\}\cup \Psi_{Z_{j_k}}.
\]
Recall that, by definition, $\hat{Z}_{\sigma(j_k)}$ is solely a function of a subset of $\{Z_{j_k}\}\cup\Psi_{Z_{j_k}}$. By the equation above, $\hat{Z}_{\sigma(j_k)}$ is solely a function of a subset of $\{Z_i\}\cup\Psi_{Z_i}$. Since we have $\pi(i)=\sigma(j_k)$ by Eq. \eqref{eq:proof_identifiability_permutation_indices}, $\hat{Z}_{\pi(i)}$ is solely a function of a subset of $\{Z_i\}\cup\Psi_{Z_i}$. Furthermore, we have $\frac{\partial\hat{Z}_{\pi(i)}}{\partial Z_i}=\frac{\partial\hat{Z}_{\sigma(j_k)}}{\partial Z_i}\neq 0$ by Eq. \eqref{eq:proof_identifiability_nonzero_derivative}, which implies that $\hat{Z}_{\pi(i)}$ is a function of $Z_i$. Therefore, $\hat{Z}_{\pi(i)}$ is solely a function of $Z_i$ and a subset of $\Psi_{Z_i}$.
\end{proof}

\section{Proof of \cref{theorem:identifiabiltiy_ICA}}\label{sec:proof:identifiabiltiy_ICA}
\ThmIdentifiabilityICA*
\begin{proof}
Suppose by contradiction that $(\hat{g}, \hat{f},p_{\hat{Z}},\hat{\Theta})$ achieves Eq. (\ref{eq:matched_distribution}). By assumption, the components of $\hat{Z}$ are independent in each domain, indicating that the Markov network $\mathcal{M}_{\hat{Z}}$ is an empty graph. By the same reasoning in the proof of \cref{theorem:identifiabiltiy_markov_network} (specifically \cref{lemma:super_graph}), there exists a permutation $\pi$ such that: if $Z_i$ and $Z_j$ are adjacent in $\mathcal{M}_Z$, then $\hat{Z}_{\pi(i)}$ and $\hat{Z}_{\pi(j)}$ are adjacent in $\mathcal{M}_{\hat{Z}}$. Since $\mathcal{M}_{\hat{Z}}$ is an empty graph, this implies that $\mathcal{M}_{Z}$ is also an empty graph. Under Assumptions~\ref{assump:adjacency_faith} and \ref{assump:sucf}, \cref{proposition:moral_graph} indicates that the undirected graph defined by $\mathcal{M}_Z$ is the moralized graph of $\mathcal{G}_Z$. Therefore, the moralized graph of $\mathcal{G}_Z$, and thus $\mathcal{G}_Z$ itself, are empty, contradicting the assumption that $\mathcal{G}_{Z}$ is not an empty graph.
\end{proof}

\section{Proof of \cref{lem:moral_sub} and \cref{proposition:moral_graph}}\label{sec:proof_thm:moral_sub}

\LemMoralSub*

\begin{proof}
Let $Z_j$ and $Z_k$, $j \neq k$ be two variables that are not adjacent in the moralized graph of $\mathcal{G}_Z$. Then it suffices to show that $\{Z_j,Z_k\} \notin \mathcal{E}(\mathcal{M}_Z)$. Because they are not adjacent in the moralized graph of $\mathcal{G}_Z$, they must not be adjacent in $\mathcal{G}_Z$ and must not share a common child in $\mathcal{G}_Z$. Thus, $Z_j$ and $Z_k$ are d-separated conditioning on $Z_{[n]\setminus\{j,k\}}$, which implies the conditional independence $Z_j\independent Z_k \mid Z_{[n]\setminus\{j,k\}}$ based on the Markov assumption on $(\mathcal{G}_Z, P_{Z;\theta})$. Then we have $\{Z_j,Z_k\} \notin \mathcal{E}(\mathcal{M}_Z)$.
\end{proof}

\ThmMoralGraph*

\begin{proof} We prove both directions as follows.

\textbf{Sufficient condition.} \ \ We prove it by contradiction. Suppose that the structure defined by $\mathcal{M}_Z$ is not equivalent to the moralized graph of $\mathcal{G}_Z$. Then, according to \cref{lem:moral_sub}, there exists a pair of variables $Z_j$ and $Z_k$, $j\neq k$ that are adjacent in the moralized graph but $\{Z_j,Z_k\} \notin \mathcal{E}(\mathcal{M}_Z)$. Thus, we have $Z_j\independent Z_k \mid Z_{[n]\setminus\{j, k\}}$. Then we consider the following two cases:
\begin{itemize}
    \item If variables $Z_j$ and $Z_k$ correspond to a pair of neighbors in $\mathcal{G}_Z$, then they are adjacent. Together with the conditional independence relation $Z_j\independent Z_k \mid Z_{[n]\setminus\{j, k\}}$, this implies that the SAF assumption is violated.
    \item If variables $Z_j$ and $Z_k$ correspond to a pair of non-adjacent spouses in $\mathcal{G}_Z$. Then they have an unshielded collider, indicating that the SUCF assumption is violated.
\end{itemize}

\textbf{Necessary condition.} \ \ We prove it by contradiction. Suppose SUCF or SAF is violated, we have the following two cases:
\begin{itemize}
    \item Suppose SUCF is violated, i.e., there exists an unshielded collider $Z_j \rightarrow Z_i \leftarrow Z_k$ in the DAG $\mathcal{G}_Z$ such that $Z_j\independent Z_k \mid Z_{[n]\setminus\{j, k\}}$. This conditional independence relation indicates that $\{Z_j,Z_k\} \notin \mathcal{E}(\mathcal{M}_Z)$. Since $Z_j$ and $Z_k$ are spouses, there exists an edge between them in the moralized graph of $\mathcal{G}_Z$, but is not contained in the structure defined by $\mathcal{M}_Z$, showing that they are not the same.
    \item Suppose SAF is violated, i.e., there exists a pair of neighbors $Z_j$ and $Z_k$ in the DAG $\mathcal{G}_Z$ such that $Z_j\independent Z_k \mid Z_{[n]\setminus\{j, k\}}$. This conditional independence relation indicates that $\{Z_j,Z_k\} \notin \mathcal{E}(\mathcal{M}_Z)$. Because $Z_j$ and $Z_k$ are adjacent in $\mathcal{G}_Z$, clearly they are also adjacent in the moralized graph of $\mathcal{G}_Z$. However, the edge between them is not contained in the structure defined by $\mathcal{M}_Z$, showing that they are not the same.
\end{itemize}
Thus, when SUCF or SAF is violated, the structure defined by $\mathcal{M}_Z$ is the moralized graph of the true DAG $\mathcal{G}_Z$.
\end{proof}

\end{document}